\documentclass{article}
\usepackage[a4paper, left=80pt, right=80pt, headsep=16pt, footskip=28pt]{geometry}
\usepackage{color,xcolor,graphicx}
\definecolor{darkgreen}{rgb}{0,0.4,0}

\usepackage{hyperref}
\hypersetup{
	breaklinks=true,
	colorlinks=true,
	linkcolor=darkgreen,
	citecolor=darkgreen,
	urlcolor=black,
}
\usepackage{amsfonts, amssymb, amsmath, amsthm, amstext, mathtools}
\usepackage{caption, subcaption, float}

\usepackage[]{authblk}

\usepackage{enumerate, enumitem, moreenum}
\usepackage[numbers,sort&compress]{natbib}  
\usepackage{dsfont}  

\usepackage[noabbrev,capitalise,nosort,nameinlink]{cleveref}

\numberwithin{equation}{section}
\numberwithin{figure}{section}
\numberwithin{table}{section}

\usepackage{parskip}

\theoremstyle{plain}
\newtheorem{theorem}{Theorem}[section]
\newtheorem{proposition}[theorem]{Proposition}
\newtheorem{lemma}[theorem]{Lemma}
\newtheorem{corollary}[theorem]{Corollary}
\theoremstyle{definition}
\newtheorem{definition}[theorem]{Definition}

\newtheorem{remark}[theorem]{Remark}

\newtheorem{assumption}[theorem]{Assumption}

\renewcommand{\epsilon}{\varepsilon}

\newcommand{\yurinsky}[1]{{\color{black}#1}}

\newcommand{\camera}[1]{{\color{black}#1}}

\newcommand{\absval}[1]{\lvert #1 \rvert}
\newcommand{\innerprod}[2]{\langle #1 , #2 \rangle}
\newcommand{\norm}[1]{\lVert #1 \rVert}

\newcommand{\Absval}[1]{\left\vert #1 \right\vert}

\newcommand{\Norm}[1]{\left\Vert #1 \right\Vert}

\newcommand*{\quark}{\setbox0\hbox{$x$}\hbox to\wd0{\hss$\cdot$\hss}}

\newcommand*{\R}{\mathbb{R}}
\newcommand*{\Rp}{\mathbb{R}_{>0}}
\newcommand*{\N}{\mathbb{N}}
\newcommand*{\E}{\mathbb{E}}
\renewcommand*{\P}{\mathbb{P}}

\renewcommand{\P}{\mathbb{P}}

\newcommand{\X}{\mathcal{X}}

\newcommand{\cX}{\mathcal{X}}

\newcommand{\dd}{\mathrm{d}}

\DeclareMathOperator{\tr}{tr}

\newcommand{\bounded}{\mathcal{L}}

\newcommand{\schatten}{\mathcal{S}}

\DeclareMathOperator*{\argmin}{arg\,min} 

\DeclareMathOperator{\idop}{Id}

\newcommand{\cov}{\mathcal{C}_\pi}
\newcommand{\empcov}{\widehat{\mathcal{C}}_\pi}
\newcommand{\regcov}{(\cov + \alpha \idop_{\mathcal{H}})^{-1}}
\newcommand{\empregcov}{(\empcov + \alpha \idop_{\mathcal{H}})^{-1}}

\newcommand{\inc}{I_\pi}
\newcommand{\empinc}{\widehat{I}_\pi}

\usepackage{colortbl}

\usepackage{mathtools}

\newcommand{\mtc}{\mathcal}

\let\originalleft\left
\let\originalright\right
\renewcommand{\left}{\mathopen{}\mathclose\bgroup\originalleft}
\renewcommand{\right}{\aftergroup\egroup\originalright}

\def\cB{{\mtc{B}}}
\def\cC{{\mtc{C}}}

\def\cH{{\mtc{H}}}

\def\cL{{\mtc{L}}}

\def\cN{{\mtc{N}}}

\def\cT{{\mtc{T}}}

\def\cX{\mathcal{X}}

\newcommand{\mbe}{\mathbb{E}}

\newcommand{\mbn}{\mathbb{N}}

\newcounter{nbnotes}
\makeatletter
\newcommand{\checknbnotes}{
\ifnum \thenbnotes > 0
\@latex@warning@no@line{**********************************************************************}
\@latex@warning@no@line{* The document contains \thenbnotes \space  note(s)}
\@latex@warning@no@line{**********************************************************************}
\fi}

\makeatother

\title{Regularized least squares learning \\
with heavy-tailed noise is minimax optimal}

\author[1]{Mattes Mollenhauer\thanks{\href{mailto:mattes.mollenhauer@merantix-momentum.com}{mattes.mollenhauer@merantix-momentum.com}}}
\author[2]{Nicole Mücke\thanks{\href{mailto:nicole.muecke@tu-braunschweig.de}{nicole.muecke@tu-braunschweig.de}}}
\author[3]{Dimitri Meunier\thanks{\href{dimitri.meunier.21@ucl.ac.uk}{dimitri.meunier.21@ucl.ac.uk}}}
\author[3]{Arthur Gretton\thanks{\href{arthur.gretton@gmail.com}{arthur.gretton@gmail.com}}}
\affil[1]{Merantix Momentum}
\affil[2]{Institute for Mathematical Stochastics, Technische Universit\"at Braunschweig}
\affil[3]{Gatsby Computational Neuroscience Unit, University College London}

\date{\today}

\begin{document}

\maketitle

\begin{abstract}

    This paper examines the performance of ridge regression in reproducing kernel Hilbert spaces in the presence of noise that exhibits a finite number of higher moments. We establish excess risk bounds consisting of 
    subgaussian and polynomial terms based on the well known integral operator framework. 
    The dominant subgaussian component 
    allows to achieve convergence rates that have previously only been
    derived under subexponential noise---a prevalent assumption
    in related work from the last two decades.
    These rates are optimal under standard eigenvalue decay conditions,
    demonstrating the asymptotic 
    robustness of regularized least squares against heavy-tailed noise.
    Our derivations are based on a Fuk--Nagaev inequality for Hilbert-space valued random variables. \\
    \textbf{Keywords.} 
    Nonparametric regression $\bullet$ 
    kernel ridge regression $\bullet$
    heavy-tailed noise
    \\
    \textbf{2020 Mathematics Subject Classification.}
    62G08 $\bullet$ 
    62G35 $\bullet$ 
    62J07
\end{abstract}


\section{Introduction}

Given two random variables $X$ and $Y$, we seek to empirically minimize the 
expected squared error
\begin{equation*}
    \mathcal{R}(f):= \mathbb{E}\left[ (Y - f(X))^2\right]
\end{equation*}
over functions $f$ in a \textit{reproducing kernel Hilbert space} 
$\mathcal{H}$ consisting of functions
from a topological space $\cX$ to $\R$.
We consider the standard model
\begin{equation*}
    Y = f_\star(X) + \epsilon
\end{equation*}
with the \emph{regression function} 
$f_\star: \mathcal{X} \to \R$ 
and noise variable $\epsilon$ satisfying $\E[\epsilon \vert X] = 0$.
Given $n$ independent sample pairs $(X_i, Y_i)$ drawn from the joint
distribution of $X$ and $Y$, we investigate
the classical \emph{ridge regression estimate}
\begin{equation}
    \label{eq:f_alpha}
    \widehat{f}_\alpha := 
    \argmin_{f \in \mathcal{H}}
    \frac{1}{n} \sum_{i=1}^n ( Y_i - f(X_i) )^2
    + \alpha \norm{f}^2_{\mathcal{H}}
\end{equation}
with \emph{regularization parameter} $\alpha > 0$.
We adopt the well-known perspective
going back to the pathbreaking work
\citep{Cucker02, devitoetal2005learning, caponnetto2007optimal, Smale2007},
which characterizes $\widehat{f}_\alpha$ as the 
solution of a linear inverse problem in
$\mathcal{H}$ obtained by performing
\emph{Tikhonov regularization} \cite{TA77}
on a stochastic discretization of the integral 
operator induced by the kernel of $\mathcal{H}$ and
the marginal distribution of $X$.
Since its inception, this approach has been refined and generalized 
in a multitude of ways, including
more general learning settings 
and alternative algorithms and applications.
We refer the reader to
\citep{BauerEtAl2007,YaoEtAl07, 
Dicker2017, Muecke18, Lin2020, fischer2020sobolev,Muecke2018Parallelizing,Szabo2016Learning,RudiRandom2017,meunier2024optimal,li2024towards, lietal2022optimal,LinBoosting2019,LinZhouDistributed2018,GerfoEtAl08,mollenhauer2022learning,singh2019kernel,meunier2024nonparametric}
and the references therein for an overview.
A common theme in the above line of work is the derivation of
confidence bounds of the \emph{excess risk}
\begin{equation*}
    \label{eq:excess_risk}
    \mathcal{R}(\widehat{f}_\alpha) - \mathcal{R}(f_\star)
    = \E [ (\widehat{f}_\alpha(X) - f_\star(X))^2 ]
\end{equation*}
i.e., with high probability
over the draw of the sample pairs
under appropriate regularity assumptions about 
the regression function $f_\star$ and distributional
assumptions about $\epsilon$. 

\paragraph{Heavy-tailed noise.}
In this work, we assume that the real-valued random variable $\epsilon$ 
has only a finite number of higher conditional 
absolute moments, i.e., 
there exists some $q \in \N$, $q \geq 3$ such that 
\begin{equation}
    \label{eq:mom_intro}
    \E\left[\absval{\epsilon}^q\right \vert X] < Q < \infty
    \text{ almost surely}.
\end{equation}
This setting covers noise associated with
distributions without a moment generating function---for example 
the $t$-distribution, Fréchet distribution, Pareto distribution
and Burr distribution (correspondingly centered).
In such a setting, the family of \emph{Fuk--Nagaev inequalities} 
\citep{Fuk1973, Nagaev1979} provides
sharp nontrivial tail bounds beyond Markov's inequality for
sums of heavy-tailed real random variables.
These results show that the tail 
is dominated by a subgaussian term \cite{Vershynin2018High} 
in a small deviation
regime (reflecting the central limit theorem)
and a polynomial term in a large deviation regime.
In order to apply this fact to the integral
operator approach, we modify a vector-valued
version of the Fuk--Nagaev inequality going back 
to \citep{Yurinsky1995} for random variables taking values
in Hilbert spaces.
\camera{
In a practical context, heavy-tailed 
noise satisfying the moment condition
\eqref{eq:mom_intro} plays a role in fields such as
finance, insurance,
communication networks and atmospherical sciences
\cite{nair2022heavy, embrechts1997extremal}. 
}

\paragraph{Prior work: Bernstein condition.}
In the aforementioned context of spectral regularization algorithms in kernel 
learning, existing work generally 
assumes that $\epsilon$ is subexponential.\!\footnote{
The term
\emph{subexponential}
in this work refers to light-tailed distributions in
the Orlicz sense \cite{Vershynin2018High} in contrast to alternative
definitions for heavy-tailed distributions found in the literature
\cite{nair2022heavy}.
}
In particular, the so-called \emph{Bernstein condition}\footnote{%
Condition \eqref{eq:bernstein} is in fact related to the  slightly stronger
\emph{subgamma property}, see \cite{boucheron2016concentration}. 
Applying Stirling's approximation to the right hand side of
\eqref{eq:bernstein} gives the typical subexponential bound for $L^q$-norms of $\epsilon$, see \cite{Vershynin2018High}.
} 
requires the existence of constants $\sigma, Q > 0$ 
almost surely satisfying
\begin{equation}
    \label{eq:bernstein}
    \E[ \absval{\epsilon}^q \vert X ] \leq \frac{1}{2}q!\sigma^2Q^{q-2} 
    \text{ almost surely}
\end{equation}
for all $q \geq 2$. This condition allows
to apply a Hilbert space Bernstein inequality
\citep{Pinelis1986Remarks} to the well-known
integral operator framework in order to obtain convergence results.
We refer the reader to
\cite{caponnetto2007optimal,BauerEtAl2007,YaoEtAl07,Muecke18,Lin2020, fischer2020sobolev, Szabo2016Learning, Dicker2017}
for a selection of results in this setting.
To our knowledge, all results obtaining
optimal rates in this setting rely on the Bernstein tail bound.
The importance of the Bernstein inequality in the context of
this work is emphasized by the
\emph{effective dimension} \citep{zhang05learning, caponnetto2007optimal}, 
which measures the capacity of the hypothesis space $\mathcal{H}$
relative to the choice of the regularization parameter $\alpha$
and the marginal distribution of $X$ in terms of
the eigenvalues of the integral operator.
When used as a variance proxy in the Bernstein inequality, 
the effective dimension is the central tool that allows to derive minimax optimal rates
under assumptions about the eigenvalue decay,
as first shown by \citep{caponnetto2007optimal} and subsequently
refined in the aforementioned work.
Due to this elegant connection between eigenvalue decay and concentration, 
the integral operator formalism has been predominantly 
focused around the assumption \eqref{eq:bernstein} over the last two decades.
Similar approaches based on the Bernstein inequality with suitable
variance proxies are commonly applied 
across a variety of estimation techniques in order to obtain
fast rates \cite[e.g.][]{gyoerfi2002regression,StCh08}.

\paragraph{Overview of contributions.}

In this work, we show that the rates derived
under the Bernstein condition \eqref{eq:bernstein}
in the mentioned literature can 
equivalently be obtained with the significantly less restrictive 
higher moment assumption \eqref{eq:mom_intro}
with the same regularization parameter schedules. 
We consider both
the \emph{capacity independent} setting (i.e., without 
assumptions about the eigenvalue decay of the integral operator \cite{BauerEtAl2007,Smale2007}) and the more common \emph{capacity dependent} setting involving 
the effective dimension \cite[e.g.][]{caponnetto2007optimal,Steinwart09Optimal,fischer2020sobolev}.
{\color{black}
Even though the capacity dependent results are sharper,
we dedicate a separate discussion to the capacity independent setting, 
as it allows a less technical presentation
and a simplified and insightful asymptotic dicussion.}
In both settings, we base our analysis on
a Hilbert space version of the Fuk--Nagaev inequality,
providing excess risk bounds that exhibit
both subgaussian and polynomial tail components.
The dominant subgaussian term
allows to asymptotically recover the familiar bounds known from the
subexponential noise scenario.
In the capacity dependent setting, we
use the effective dimension not only as variance proxy, but
also as a \emph{proxy for the higher moments occurring in the
Fuk--Nagaev inequality}---the resulting bound is sharp enough
so that standard assumptions about the eigenvalue decay lead to known
optimal convergence rates. This technique directly 
generalizes the aforementioned approach based on the Bernstein inequality.

\paragraph{Practical implications, future work and limitations.}

\camera{
The square loss is often not used in practice when one expects heavy-tailed noise,
as it is sensitive to outliers when used without regularization
\cite{huber1981robust}.
However, when used with regularization in the presence of 
noise of the form \eqref{eq:mom_intro},
we show that it exhibits a certain degree of robustness. 
In particular, 
\begin{enumerate}[label=(\roman*)]
    \item it asymptotically achieves the optimal rates with high probability 
    known from the light-tailed setting with the same
    regularization schedule, 
    \item in a large sample setting, 
    the confidence behavior of the excess risk is essentially subexponential,
    \item in a small sample setting, the confidence behavior 
    is polynomial and stronger regularization is required 
    due to the impact of the heavy tails.
\end{enumerate}
}
We focus on the original well-specified kernel ridge regression 
setting as investigated by \citep{caponnetto2007optimal} in order to simplify 
the presentation and highlight the key arguments. However, we expect
our approach to transfer to other settings, for example involving
more general source conditions \citep{BauerEtAl2007,Rastogi2017},
general spectral filter methods \citep{GerfoEtAl08,Muecke18}, 
misspecified models \citep{fischer2020sobolev,zhang2023},
the kernel conditional mean embedding with unbounded kernels
on the target space \citep{lietal2022optimal},
high- and infinite-dimensional output spaces \citep{li2024towards,meunier2024optimal}
and many other settings allowing for the application of the
integral operator formalism.
\camera{
We believe that kernel regression with unbounded kernels 
admitting a finite higher moment can
be analyzed with a similar technical approach as the one presented here.
Let us mention some limitations of the present work.
While our results show a certain degree of robustness of
regularized least squares against heavy-tailed noise for $q \geq 3, q \in \N$,
we expect our results to directly transfer to all real $q > 2$, as versions
of the real-valued Fuk--Nagaev bound cover this case \cite{FukNagaev1971, Rio2017}. 
Currently, this restriction of our results
exclusively depends on the validity of the Fuk--Nagaev bound in Hilbert spaces for these
$q$ as an artefact of the proof technique by \cite{Yurinsky1995},
which we modify.
Furthermore, for $q < 2$, it is clear that the square loss is not
necessarily well-defined and a different loss such as the
Cauchy loss should be used \cite{wen2025robustness}. 
Our results demonstrate that, when operating in a high confidence setting,
heavy-tailed noise may require a significantly higher level of regularization than
light-tailed noise, making empirical regularization parameter
choice rules as a function of the confidence level very important.
Extending the analysis of classical parameter choice rules for deterministic inverse problems
\cite{engl1996inverse} to the stochastic setting based on heavy-tailed noise
may therefore be an interesting future direction.
}

\paragraph{Other related work.}

We are not aware of any specific analysis of 
kernel ridge regression and the integral operator formalism
in the heavy-tailed scenario given by \eqref{eq:mom_intro} in the literature.
However, there exists a wide variety of related results 
for regression
with heavy-tailed noise and robust estimation---we put our results 
in the context of the most important related work.
Optimal rates for (unpenalized) least squares regression
over nonparametric hypothesis spaces can generally only be derived 
in the empirical process context when $q$ in 
condition \eqref{eq:mom_intro}
is large enough with respect to suitable metric entropy requirements, 
see \cite{han2019convergence, kuchibotla2022leastsquares, mendelson2016multiplier, 
mendelson2020lp}.
In comparison, our setting allows to recover optimal rates 
for the reproducing kernel Hilbert space scenario
with a regularization schedule which is independent of $q$.
For linear models over finite basis functions, \cite{audibert2011robust}
derives exponential concentration of the ridge estimator under finite
variance on the noise for fixed $\alpha$. 
We also highlight the field of
robust estimation techniques outside of the standard least squares context, 
see e.g.\
\cite{huber1981robust,hsu2016loss, brownlees2015empirical, lugosi2019mean, catoni2012challenging, audibert2011robust}
and the references therein. 
While the analysis
of robust finite-dimensional linear regression under heavy-tailed noise requires
discussions of the distribution of the covariates
and their covariance matrix (often under variance-kurtosis equivalence
\cite{lugosi2019subgaussian,mendelson2020equivalence, mourtada2021robust}),
we impose the typical assumption that the
kernel is bounded, leading to subgaussian concentration
of the embedded covariates
in the potentially infinite-dimensional feature space.
Recently, \cite{wen2025robustness}
derived nearly optimal rates for kernel ridge regression with Cauchy loss
under \eqref{eq:mom_intro} with $q > 0$
depending on the Hölder continuity parameter of the target function
using more classical arguments. 
Finally, from a more technical perspective, 
the approach by \cite{zhu2022beyond} 
shares similarities with the methods applied in our paper:
The authors use a real-valued Fuk--Nagaev inequality
to bound the stopping time complexity of stochastic gradient descent for ordinary least squares regression.
\camera{
However, \cite{zhu2022beyond} provides results only
for the finite-dimensional setting
based on a martingale decomposition by 
assuming a lower bound on the minimal eigenvalue of the covariance
matrix. In contrast, our work targets the infinite-dimensional setting based on inverse problem theory without such a lower bound and explicitly proves minimax optimality.
}

\paragraph{Structure of this paper.}
We introduce our notation and basic preliminaries in \cref{sec:preliminaries}.
In \cref{sec:ridgre_regression}, we provide the excess risk bound
for the capacity independent setting and derive corresponding rates.
\cref{sec:capacity_dependent} contains an excess risk bound
based on the effective dimension and recovers 
rates which are known to be minimax optimal also in 
the subexponential noise setting.
Finally, in \cref{sec:fuk-nagaev}, we briefly discuss
the Fuk--Nagaev inequality used to derive our results.
\camera{
\cref{sec:experiment} contains
a numerical experiment which confirms the
behavior of the excess risk described by our theoretical results.
}
We report all proofs for the results in the main text
in \cref{app:proofs} and provide additional technical results as
individual appendices.


\section{Preliminaries}
\label{sec:preliminaries}

We assume that the reader is familiar
with the analysis of linear 
Hilbert space operators \citep{weidmann1980linear, reed}
and the basic theory of reproducing kernel Hilbert spaces
\citep{Berlinet04:RKHS, StCh08}.
Let $\pi$ denote the marginal distribution of $X$
and $L^2(\pi)$ denote the space of real-valued
Lebesgue square integrable functions with respect to $\pi$.
We write $\bounded(H_1, H_2)$ for the space of bounded linear
operators between Hilbert spaces $H_1$ and $H_2$
with operator norm $\norm{\cdot}$ and abbreviate
$\bounded(H_1) = \bounded(H_1, H_1)$. We additionally consider the space of
\emph{Hilbert--Schmidt operators}
$\schatten_2(H_1, H_2) \subset  \bounded(H_1, H_2)$
and the space of \emph{trace class operators}
$\schatten_1(H_1, H_2) \subset \schatten_2(H_1, H_2)$
with norms $\norm{\cdot}_{\schatten_1(H_1, H_2)}$
and $\norm{\cdot}_{\schatten_2(H_1, H_2)}$ and the \emph{trace} $\tr(\cdot)$.
The \emph{adjoint} of $A \in\bounded(H_1, H_2)$
is written as $A^* \in\bounded(H_2, H_1)$.

\subsection{Reproducing kernel Hilbert space}

We consider the \emph{reproducing kernel Hilbert space} (RKHS)
$\mathcal{H}$ 
consisting of functions from $\X$ to $\R$
induced by the \emph{symmetric positive semidefinite kernel}
$
    k \colon \X \times \X \to \R.
$
with \emph{canonical feature map} 
\begin{align*}
    \phi(x): \cX \to \mathcal{H}, \\
    x \mapsto k(\cdot, x),
\end{align*}
i.e.\ we have the \emph{reproducing property}
$f(x) = \innerprod{f}{\phi(x)}_{\mathcal{H}}$ for all $x \in \mathcal{X}$ and $f \in \mathcal{H}$.

\begin{assumption}[Domain and kernel]
\label{assump:kernel}
We impose the following standard assumptions throughout this paper
in order to avoid issues related to measurability
and integrability \citep[Section 4.3]{StCh08}:
\begin{enumerate}[label=(\roman*)]
    \item $\cX$ is a second-countable locally compact Hausdorff space, $\mathcal{H}$ is separable
    (this is satisfied if $k$ is continuous, given that $\cX$ is separable),
    \item $k(\cdot, X)$ is almost surely measurable in its first argument,
    \item $k(X,X) \leq \kappa^2$ almost surely for some finite constant $\kappa$.
\end{enumerate}
\end{assumption}

\camera{
All assumptions above hold for commonly used continuous
radial kernels on $\R^d$ such as the Gaussian kernel and the Mat\`ern kernel.
However, the boundedness assumption is violated for
polynomial kernels unless $X$ is bounded almost surely.
}

\paragraph{Integral and covariance operators.}

Under \cref{assump:kernel}, we may consider the 
typical linear operators
associated with $\pi$ and $k$.
We consider the \emph{embedding operator}
\begin{align*}
    I_\pi \colon \mathcal{H} &\hookrightarrow L^2(\pi),\\
    f &\mapsto [f]_\pi
\end{align*}
identifying
$f \in \mathcal{H}$ with its equivalence class $[f]_\pi \in L^2(\pi)$.
The adjoint $I^*_\pi: L^2(\pi) \to \mathcal{H}$
is given by 
\begin{equation*}
    I^*_\pi f 
    = 
    \int_\cX \phi(x) f(x) \, \dd \pi(x) \in \mathcal{H},
    \quad f \in L^2(\pi).
\end{equation*}

We obtain the self-adjoint \emph{integral operator} 
$\mathcal{T}_\pi:= I_\pi I_\pi^* : L^2(\pi) \to L^2(\pi)$ induced by
$k$ and $\pi$ as
\begin{equation*}
    \mathcal{T}_\pi f = 
    \int_\cX  k(\cdot,x) f(x) \, \dd \pi(x),
    \quad f \in L^2(\pi).
\end{equation*}
The self-adjoint \emph{kernel covariance operator}
$\mathcal{C}_\pi:= I_\pi^* I_\pi : \mathcal{H} \to \mathcal{H}$
is given by
\begin{equation*}
    \mathcal{C}_\pi = 
    \int_\cX \phi(x) \otimes \phi(x) \, \dd \pi(x).
\end{equation*}
\cref{assump:kernel}(iii)
ensures that we have
$
    \norm{I_\pi}
    =
    \norm{I^*_\pi}
    \leq \kappa
$
as well as
$\norm{\mathcal{T}_\pi}
= 
\norm{\mathcal{C}_\pi}
\leq \kappa^2
$.
Moreover, the operators $I_\pi$ and $I^*_\pi$ are Hilbert--Schmidt
and both $\mathcal{T}_\pi$ and $\mathcal{C}_\pi$ are therefore self-adjoint positive semidefinite and
trace class.
By the \emph{polar decomposition} of $\inc$ and $\inc^*$, there exist partial isometries $U: \mathcal{H} \to L^2(\pi)$ 
and $\tilde U:  L^2(\pi) \to \mathcal{H}$ such that
\begin{equation}
    \label{eq:polar_decomposition}
    \inc = U (I_\pi^* I_\pi)^{1/2} = U \mathcal{C}_\pi^{1/2}
    \text{ and }
    \inc^* = \tilde U (I_\pi I_\pi^*)^{1/2} = \tilde U \mathcal{T}_\pi^{1/2}.
\end{equation}
In particular, we have 
$\norm{[f]_\pi}_{L^2(\pi)} = \norm{I_\pi f}_{L^2(\pi)} = \norm{\mathcal{C}^{1/2}f}_{\mathcal{H}}$
for all $f \in \mathcal{H}$. We will
also frequently use the fact that \cref{assump:kernel}(iii)
implies $\sup_{x \in \mathcal{X}}{|f(x)|}\leq \kappa \norm{f}_{\mathcal{H}}$.


\subsection{Ridge regression}

We introduce the standard integral operator formalism for ridge regression 
in RKHSs, see e.g.\ \cite{caponnetto2007optimal}.
We consider the standard $L^2(\P)$-orthogonal decomposition 
of $Y$ with respect to the closed subspace $L^2(\P, \sigma(X)) \subset 
L_2(\P)$ of $\sigma(X)$-measurable functions given by
\begin{equation}
    \label{eq:y_decomposition}
    Y = f_\star(X) + \epsilon
\end{equation}
with the \emph{regression function} 
$f_\star(X) = \E[Y \vert X] \in L_2(\P)$ 
and noise variable $\epsilon \in L^2(\P)$ satisfying $\E[\epsilon \vert X] = 0$.
Based on the representation \eqref{eq:y_decomposition}, we have
$
    \mathcal{R}(f) = \norm{ f - f_\star }^2_{L^2(\pi)} + \norm{ \epsilon }^2_{L^2(\P)}
$
for all $f \in L^2(\pi)$ and hence
the excess risk satisfies
$
    \mathcal{R}(f) - \mathcal{R}(f_\star) = \norm{ f - f_\star }^2_{L^2(\pi)}.
$

\paragraph{Regularized population solution.}
We define the \emph{regularized population solution}
\begin{align*}
    f_\alpha :&= 
    \argmin_{f \in \mathcal{H}}
    \E[  (Y - f(X))^2 ] + \alpha \norm{f}^2_{\mathcal{H}} 
    =
    \argmin_{f \in \mathcal{H}}
    \norm{ \inc f - f_\star }^2_{L^2(\pi)} + \alpha \norm{f}^2_{\mathcal{H}},
\end{align*}
with $\alpha >0$, which is alternatively expressed as
\begin{equation}
    \label{eq:solution_rkhs}
    f_\alpha =
    \regcov I_\pi^* f_\star
    \in \mathcal{H}
\end{equation}
with the identity operator $\idop_{\mathcal{H}}$ on $\mathcal{H}$.

\paragraph{Regularized empirical solution.}

We consider sample pairs
$(X_1, Y_1), \dots, (X_n, Y_n) \sim \mathcal{L}(X,Y)$
independently obtained from the joint distribution of $X$ and $Y$.
We define the empirical versions of the
operators above in terms of 
\begin{equation*}
    \widehat{I}^*_\pi f 
    := 
    \frac{1}{n} \sum_{i=1}^n \phi(X_i) f(X_i) \in \mathcal{H} ,
    \quad f \in L^2(\pi)
\end{equation*}
as well as
\begin{equation*}
    \widehat{\mathcal{C}}_\pi 
    :=
    \frac{1}{n} \sum_{i=1}^n \phi(X_i) \otimes \phi(X_i).
\end{equation*}
The \emph{empirical solution} of the learning problem 
in $\mathcal{H}$ with regularization parameter $\alpha >0$ 
is given by the empirical analogue of \eqref{eq:solution_rkhs}, 
which we obtain in terms of 
\begin{equation}
    \label{eq:empirical_solution}
    \widehat{f}_\alpha = 
    \empregcov \widehat{\Upsilon}
    \in \mathcal{H},
\end{equation} where the empirical
right hand side $\widehat{\Upsilon} \in \mathcal{H}$
of the inverse problem is given by
\begin{equation}
    \label{eq:empirical_rhs}
    \widehat{\Upsilon} 
    := \frac{1}{n} \sum_{i=1}^n \phi(X_i) Y_i
    =
    \empinc^* f_\star
    + \frac{1}{n} \sum_{i=1}^n \phi(X_i) \epsilon_i.
\end{equation}
Here, we use the orthogonal decomposition 
$Y_i = f_\star(X_i) + \epsilon_i$ 
in the second equivalence.
We note that $\widehat{\Upsilon}$
directly serves as an empirically evaluable unbiased estimate of 
$I^*_\pi f_\star$, 
as $\widehat{I}^*_\pi f_\star$ itself cannot be empirically evaluated 
because $f_\star$ is unknown.
As usual, we interpret the above objects as random variables depending on the
product measure $\P^{\otimes n}$
through their definition based on the observation pairs $(X_i, Y_i)$.
\camera{
In practice, the empirical solution $\widehat{f}_\alpha$
can be evaluated in terms of the classical \emph{representer
theorem}, see e.g.\ \cite[Proposition 8]{Cucker02}.
}

\subsection{Distributional assumptions}

We list the assumptions we impose upon the distributions
of $Y$, $X$ and $\epsilon$.

\begin{assumption}[Moment condition]
\label{assump:mom}
We consider the model given by \eqref{eq:y_decomposition}
and assume that we have almost surely
\begin{align}
    \E[\epsilon \vert X ] = 0, \quad
    \E\left[ \epsilon^2 \vert X \right] 
        < \sigma^2
    \quad \text{and} \quad
    \E\left[\absval{\epsilon}^q\right \vert X] < Q,
    \tag{MOM}
    \label{eq:mom}
\end{align}
for some constants $\sigma^2 >0$, $Q >0$ and $q \in \N$, $q \geq 3$.
\end{assumption}

We now introduce a classical smoothness assumption
in terms of a \emph{Hölder source condition} \citep{EHN96}.

\begin{assumption}[Source condition]
    \label{assump:src}
    We define the set
    $
        \Omega(\nu, R) :=
        \{
        \mathcal{T}_\pi^\nu f \mid \norm{f}_{L^2(\pi)} \leq R
        \} \subset L^2(\pi)
    $
    and assume 
    \begin{equation}
        f_\star \in  \Omega(\nu, R)
        \tag{SRC}
        \label{eq:src}
    \end{equation}
    for some smoothness parameter $\nu \geq 1/2$ and $R > 0$.
\end{assumption}

We give the definition of the source set $\Omega(\nu, R)$
with respect to $L^2(\pi)$ and not with respect to $\mathcal{H}$,
which is also commonly found in the literature.
Furthermore, the source condition is sometimes
described in terms of so-called \emph{interpolation spaces} 
or \emph{Hilbert scales}. 
Our definition can equivalently be expressed in terms of these concepts by
appropriately reparametrizing $\nu$, 
see e.g.\ \citep{BauerEtAl2007, fischer2020sobolev, EHN96} for more details.

\begin{remark}[Well-specified case]
    In this work, we explicitly consider the
    condition $\nu \geq 1/2$, which implies the
    \emph{well-specified} setting in which we have
    $\Omega(\nu, R) \subset I_\pi(\mathcal{H})$.
    We note the case $0 \leq \nu < 1/2$ 
    covers the \emph{misspecified} setting, in which 
    $\Omega(\nu, R)$ is allowed to contain elements from $L^2(\pi) \setminus I_\pi(\mathcal{H})$.
    We expect our approach to transfer to 
    the misspecified setting
    by combining it with recent 
    technical arguments from the literature
    which are outside the scope
    of this work \citep{fischer2020sobolev, zhang2023, zhang2024, li2024towards, 
    meunier2024optimal}.
\end{remark}


\section{Capacity-free excess risk bound}
\label{sec:ridgre_regression}

We now provide an excess risk bound
and corresponding rates
for kernel ridge regression in the heavy-tailed noise setting
without additional assumptions about the
eigenvalue decay of $\mathcal{T}_\pi$.
{\color{black}
We present this setting separately from the 
capacity-based results in the next section, as it allows for a clearer
comparison with bounds based on subexponential noise and a simplified
asymptotic discussion.
}

\begin{proposition}[Main excess risk bound]
\label{prop:excess_risk_bound}
Let \eqref{eq:mom} and \eqref{eq:src} be satisfied. 
For all $\delta \in (0, 1)$ 
and $n \in \N$
such that
\begin{equation}
\label{eq:alpha-coice}
C_\kappa\; \log(6/\delta) \leq \alpha \sqrt{n}, 
\qquad C_\kappa:=2(1+\sqrt \kappa)\cdot \max\{1, \kappa^2\} ,  
\end{equation}
we have
\begin{align*}
    \norm{\inc  \widehat f_\alpha - f_\star}_{L^2(\pi)}
    &\leq R \alpha^{\min\{ \nu, 1 \}} + \frac{C_\diamond}{\sqrt \alpha} 
    \left(    \frac{\log(6/\delta)}{n} + 
  \sqrt{ \frac{\alpha^{2 \min\{\nu , 1\}} \log(6/\delta)}{n} }  
    +    \eta(\delta, n)
    \right),
\end{align*}
with confidence $1-\delta$, with
\begin{equation*}
    \eta(\delta, n) :=
    \max
    \left\{
    \left(\frac{Q}{\delta n^{q-1}} \right)^{1/q}  ,
      \sigma\; \sqrt{\frac{ \log(6c_1/\delta)}{n }}
    \right\},
\end{equation*}
where $0 < C_\diamond $ is given in 
\eqref{eq:diamond} and $c_1 \geq 1$ is the
constant from \cref{prop:fn_modified}
depending only on $q$.

\end{proposition}

    Just as in the light-tailed setting,
    \cref{prop:excess_risk_bound} shows
    that the optimal excess risk is achieved
    by balancing the contributions of the \emph{approximation error} (e.g.\ the model \emph{bias}) and
    \emph{sample error} (e.g.\ the model \emph{variance})
    by choosing a suitable regularization parameter $\alpha$
    depending on $n$ and $\delta$.
    The term $R \alpha^{\min\{ \nu, 1 \}}$ quantifies the approximation error based on the smoothness of 
    $f_\star$ and exhibits the typical \emph{saturation effect}
    of ridge regression: the fact that
    the convergence speed
    cannot be improved beyond a smoothness level $\nu=1$
    \citep[e.g.][]{neubauer1997converse,li2023on}.
    The key difference to known results for
    subexponential noise in this setting \cite{Smale2007,BauerEtAl2007} 
    is the Fuk--Nagaev term 
    $\eta(\delta, n)$ appearing in the sample error,
    which introduces an additional polynomial dependence
    on $\delta$ and $n$. We now investigate
    the consequences of this term.

\paragraph{Confidence regimes.}
We split the confidence scale into two disjoint intervals
depending on whether the subgaussian
component or the polynomial component
dominates in the term $\eta(\delta, n)$.
For $n \in \mbn$, $q\geq 3$ we define  
\begin{align}
 D_1(n, q)&:= 
\left\{ \delta \in (0,1)\;:\;   \eta(\delta , n) =  \sigma\sqrt{\frac{ \log(6c_1/\delta)}{ n} }
\right\} \nonumber\\
&= \left\{ \delta \in (0,1)\;:\; n \geq \left( \frac{Q^2}{\sigma^{2q}}\right)^{\frac{1}{q-2}} 
\delta^{-\frac{2}{q-2}}\cdot \log(6c_1/\delta)^{-\frac{q}{q-2}}\right\}  , \label{eq:D1}\\
 D_2(n, q) & := (0,1) \setminus D_1(n, q) \;.
 \label{eq:D2}
\end{align}

In what follows, we will refer to $D_1(n, q)$
as the \emph{subgaussian confidence regime}
and $ D_2(n, q)$ as the \emph{polynomial confidence regime}.
The \emph{effective sample size} $n_0$
ensuring subgaussian behavior of $\eta(n, \delta)$ 
for all $n \geq n_0$ is hence
\begin{equation}
\label{eq:effective_samples}
        n_0 := \left( \frac{Q^2}{\sigma^{2q}}\right)^{\frac{1}{q-2}} 
        \delta^{-\frac{2}{q-2}}\cdot \log(6c_1/\delta)^{-\frac{q}{q-2}}.
\end{equation}

\begin{figure}[H]
    \centering
     \includegraphics[width=0.60\textwidth]{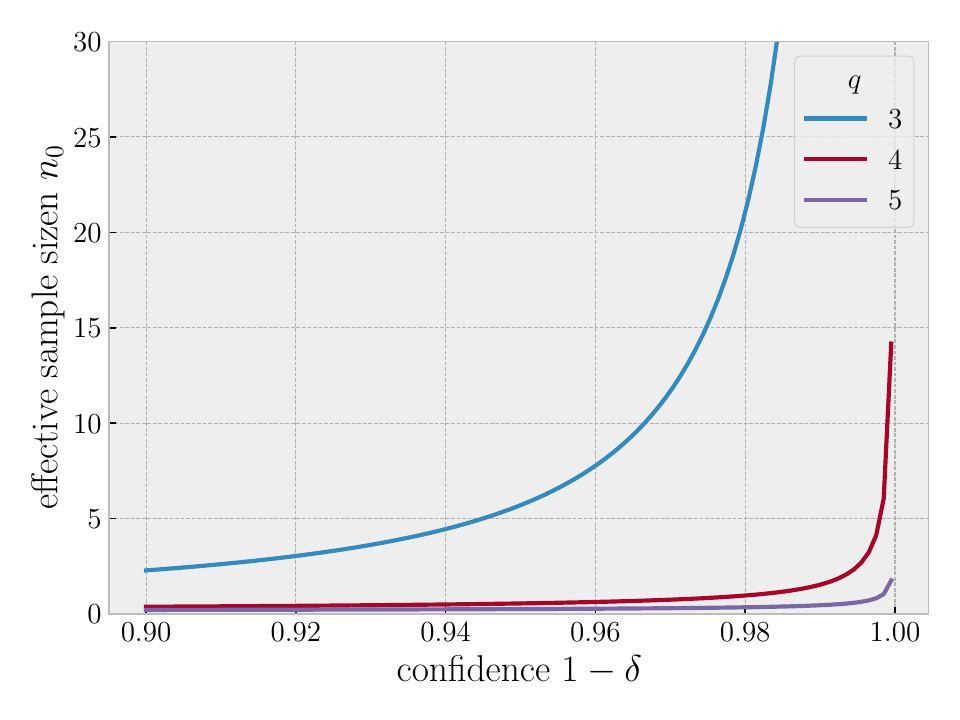}
     \caption{Illustration of the
     the effective sample size $n_0$ ensuring
     subgaussian behavior of the term
     $\eta(\delta, n)$ defined
     in \eqref{eq:effective_samples} for different choices of
     $q$. For simplicity, we set $c_1 = 1$,
     $\sigma = 2$ and $Q = 10$.
     \label{fig:sample_size}}
\end{figure}

We illustrate the behavior of $n_0$ depending on 
$1-\delta$ for different choices of $q$ in
\cref{fig:sample_size}
(we note that the involved constant $c_1$
stems from the Fuk--Nagaev inequality given in 
\cref{prop:fn_modified} and generally depends
on $q$). We choose $c_1 = 1$ for simplicity
to provide a basic intuition---note that $c_1$ only affects $n_0$
logarithmically.
We refer the reader to  \cite{Rio2017} for a detailed discussion
based on the bound for real-valued random variables.

\paragraph{Subgaussian confidence regime and convergence rates.}
We now give an excess risk bound which is similar to the
setting with bounded or subexponential noise: it
exhibits a logarithmic dependence of the confidence
parameter $\delta$ and a dependence of the sample size
up to $n^{-1/3}$ depending on 
the level of smoothness given by $\nu$ \citep{Smale2007, BauerEtAl2007}.
In the asymptotic large sample context,
this bound is dominant and allows us to recover convergence
rates.

\begin{corollary}[Subgaussian confidence regime]
\label{cor:low_confidence}
Let \eqref{eq:mom} and \eqref{eq:src} be satisfied.
Then there exist constants $\tilde{c}_1, \tilde{c}_2 > 0$ such
that with the regularization schedule
\begin{equation*}
\alpha_1(n, \delta) := \tilde{c}_2 \left( \frac{\log(6 c_1 / \delta)}{n} \right)^{\frac{1}{2\min\{\nu ,1\} +1}}
\end{equation*} 
we have
\begin{align}
    \label{eq:error_bound_regime1}
    \norm{\inc  \widehat f_{\alpha_1(n, \delta)} - f_\star}_{L^2(\pi)}
    &\leq 
    \tilde{c}_1 R \;\left( 
    \frac{\log(6c_1/\delta)}{n} \right)^{\frac{\min\{ \nu, 1 \}}{2\min\{\nu ,1\} +1}} , 
\end{align}
with confidence $1-\delta$   
for all $\delta \in D_1(n, q)$ and $n\in \N$ such that 
$\alpha_1(n, \delta) \leq \kappa^2$.
\end{corollary}

The constants $\tilde{c}_1$ and $\tilde{c}_2$ in \cref{cor:low_confidence} 
only depend on $R, \nu, \kappa, \sigma, c_1, c_2$ and can
be made explicit, but we omit their closed form 
here for the sake of a more accessible presentation.
We refer the reader to the proof 
for more details.

\begin{remark}[Convergence rates]
    We directly obtain
    convergence rates from the above consideration.
    For \emph{fixed confidence parameter} $\delta \in (0,1)$,
    we see that for all $n \geq n_0$, where
    $n_0$ 
    is the effective sample size given in \eqref{eq:effective_samples},
    we have
    $\delta \in D_1(n , q)$.
    Furthermore, 
    there exists some $\tilde{n}_0 \in \N$ such that
    $\alpha_1(n, \delta) \leq \kappa^2$ for all $n \geq \tilde{n}_0$.
    Combining these two insights, 
    from \cref{cor:low_confidence}, we obtain
    \begin{align*}
    \norm{\inc  \widehat f_{\alpha_1(n, \delta)} - f_\star}_{L^2(\pi)}
    &\leq \tilde{c}_1 R \;\left( \frac{\log(6c_1/\delta)}{n} \right)^{\frac{\min\{ \nu, 1 \}}{2\min\{\nu ,1\} +1}}
    \end{align*}
    with confidence $1-\delta$
    for all $n \geq \max\{n_0, \tilde{n}_0 \}$. We explicitly note
    that the convergence rates as well as the regularization schedule
    $\alpha_1(n, \delta)$
    match exactly the known results for the
    capacity-independent setting that have been derived under 
    the assumption of bounded or subexponential noise
    \citep{Smale2007, BauerEtAl2007}. 
\end{remark}

\vspace{0.2cm}
\paragraph{Polynomial confidence regime.}
By definition \eqref{eq:D1},
the polynomial confidence regime $\delta \in D_2(n, q)$ 
is relevant in the nonasymptotic investigation whenever
$n < n_0$.
For completeness, 
we address this setting in \cref{app:polynomial_regime}
and show that \cref{prop:excess_risk_bound} can yield
simplified risk bounds with suitable regularization schedules for
$\alpha$ based on $\delta$ and $n$.
Depending on $\delta$, these bounds may 
require a stronger regularization $\alpha_2(n, \delta)$ than
the subgaussian confidence setting. In fact, the resulting bound
exhibits a polynomial worst-case dependence on $\delta$,
which is compensated by a better dependence on the sample size before 
transitioning to the behavior from the subgaussian confidence 
regime given by \cref{cor:low_confidence}.
\camera{
This behavior can be observed in practice, which we confirm
in a basic numerical experiment provided in \cref{sec:experiment}.
}


\section{Capacity dependent bound and optimal rates}
\label{sec:capacity_dependent}

We now improve the results from the previous section and give
an excess risk bound that involves the \emph{effective dimension},
which has been established as a central tool
to quantify the algorithm-dependent capacity of the hypothesis space
$\mathcal{H}$ relative to the distribution
$\pi$ and regularization parameter $\alpha$ 
in order to derive risk bounds in regularized kernel-based learning
under the assumptions of an eigenvalue decay of $\cov$
\citep{zhang05learning, caponnetto2007optimal,Muecke18, Lin2020}.

\begin{definition}[Effective dimension]
For $\alpha >0$, we define 
\begin{equation*}
    \cN(\alpha):= \tr\left(
    \cC_\pi (\cC_\pi + \alpha \idop_\mathcal{H})^{-1}\right) < \infty.
\end{equation*}
\end{definition}

We now assume the standard polynomial eigenvalue decay of
$\cov$ 
\citep{caponnetto2007optimal, Muecke18, fischer2020sobolev}.

\begin{assumption}[Eigenvalue decay] \label{asst:evd}
    We assume that the nonincreasingly ordered sequence of 
    nonzero eigenvalues $(\mu_i)_{i \geq 1}$
    of $\mathcal{T}_\pi$ satisfies the decay
    \begin{equation}
        \label{eq:evd_new}
        \mu_i \leq D i^{-1/p}, \quad i \in \N
        \tag{EVD}
    \end{equation}    
    for a constant $D > 0$ and some $p \in (0, 1)$.
\end{assumption}

Under the additional assumption \eqref{eq:evd_new},
we can sharpen \cref{prop:excess_risk_bound}.

\begin{proposition}[Capacity-dependent excess risk bound]
\label{prop:excess_risk_capacity_new}
Let \eqref{eq:mom}, \eqref{eq:src} and \eqref{eq:evd_new} be satisfied. 
Suppose that $\delta \in (0,1)$ and
\begin{equation}
\label{eq:ass:eff:dim_new}
\log(2/\delta)\left( \frac{2\kappa^2}{n\alpha} + \frac{2 \sqrt{\tilde{D}} \kappa}{\sqrt{n} \alpha^{(1 + p)/2}} \right) \leq 1.
\end{equation}
Then there exists a constant $c >0$ not depending
on $\delta$ and $n$ such that with confidence $1-\delta$, we have
\begin{align*}
\Vert I_\pi \widehat f_\alpha -   f_\star\Vert _{L^2(\pi)} &\leq 
c \; \left(  \alpha^{\min\{\nu, 1\} } +  \frac{ \log(8/\delta)}{\sqrt{\alpha} n}  + \sqrt{\frac{\cN(\alpha) \log(8/\delta)}{n}}
+  \sqrt{\cN(\alpha)} \cdot    \eta(\delta, n, \alpha) \right), 
\end{align*}
where we set 
\begin{equation*}
    \eta(\delta, n, \alpha) 
    := \max \left\{ \left(\frac{1}{\delta n^{q-1}} \right)^{1/q}  \cdot
    \left( \frac{1}{\alpha \cN(\alpha)}\right)^{\frac{q-2}{2q}} ,
    \sqrt{\frac{\log(8c_1/\delta)}{n} }
    \right\}.
\end{equation*}
\end{proposition}

The constant $c$ is made explicit in the proof.
The key idea builds upon the original work 
of \citep{caponnetto2007optimal}.
In particular, we incorporate the effective dimension 
into the Fuk--Nagaev inequality as a proxy for the $q$-th absolute
moment appearing in the term $\eta(\delta, n, \alpha)$, 
thereby generalizing the idea to use the effective dimension 
as a variance proxy in the Bernstein inequality.

\begin{corollary}[Convergence rates]
\label{cor:capacity_rates_new}
    Let \eqref{eq:mom}, \eqref{eq:src} and
    \eqref{eq:evd_new} be satisfied. 
    Then for every $\delta \in (0,1)$, there exists
    some $n_0 \in \N$ such
    that with the regularization schedule
    \begin{equation*}
    \alpha(n, \delta) := \left( \frac{\log(8 c_1 / \delta)}{n} \right)^{\frac{1}{2\min\{\nu ,1\} +p}},
    \end{equation*}  
    we have
    \begin{align}
        \label{eq:error_bound_regime1}
        \norm{\inc  \widehat f_{\alpha(n, \delta)} - f_\star}_{L^2(\pi)}
        &\leq 
        c
        \left( \frac{\log(8 c_1 / \delta)}{n} \right)^{\frac{\min\{\nu ,1\}}{2\min\{\nu ,1\} +p}} 
    \end{align}
    with confidence $1-\delta$ for all $n \geq n_0$
    with a constant $c > 0$
    independent of $n$ and $\delta$.
\end{corollary}

\begin{remark}[Optimality of rates]
    The rates provided by \cref{cor:capacity_rates_new}
    match the rates from the literature
    derived for well-specified 
    kernel ridge regression under the assumption
    of subexponential noise
    \citep{caponnetto2007optimal,Steinwart09Optimal}. 
    In particular, these rates are known to be minimax optimal
    over the class of distributions satisfying
    \eqref{eq:src}, \eqref{eq:evd_new} and the
    Bernstein condition \eqref{eq:bernstein}.
    \cref{cor:capacity_rates_new} now 
    proves that one can significantly
    relax the assumption of subexponential noise \eqref{eq:bernstein}, as rate
    optimality is already achieved under the condition \eqref{eq:mom}.
    Furthermore, the regularization schedule
    $\alpha(n, \delta)$ is the same as in the light-tailed 
    setting---in particular, it does not depend on $q$.
\end{remark}


\section{Fuk--Nagaev inequality in Hilbert spaces}
\label{sec:fuk-nagaev}

We discuss the central ingredient for the derivation of the previous results in more detail for convenience.
We refer the reader to \citep{Fuk1973,Nagaev1979} for the original 
work in the setting of real-valued random variables and \citep{Rio2017}
for a discussion of the involved constants.
We present a sharpened version 
of a result due to \citep[][Theorem 3.5.1]{Yurinsky1995}, which is formulated more
generally in normed spaces, but exhibits an excess term
that can be removed in the Hilbert space case.
We provide the proof in \cref{app:concentration_bounds}.
\yurinsky{We also note that the proof of the
result in \cite{Yurinsky1995} is incomplete due to an inconsistent exponential moment bound. 
We address this issue by deriving an alternative bound.}

\begin{proposition}[Fuk--Nagaev inequality; Hilbert space version]
\label{prop:fn_modified}
Let $\xi, \xi_1, \dots \xi_n$ be independent and identically distributed random variables taking values
in a separable Hilbert space $\X$ such that
\begin{equation*}
    \E[\xi] = 0, \quad
    \E\left[ \norm{\xi}_\X^2 \right] < \sigma^2
    \quad \text{and} \quad
    \E\left[\norm{\xi}_\X^q\right] < Q,
\end{equation*}
for some constants $\sigma^2 >0$, $Q >0$ and $q \in \N$, $q \geq 3$.
Write $S_n :=  \sum_{i = 1} \xi_i$.
Then there exist two constants $c_1 >0$
and $c_2> 0$ depending only on $q$ such that for every $t >0$, we have
\begin{equation}
    \label{eq:fn_modified}
    \P \left[ \,
        \Norm{ \frac{1}{n} S_n }_\X  > t
        \right]
        \leq 
        c_1 \left(  
        \frac{Q}{t^q n^{q-1}}
        + 
        \exp\left(-c_2 \frac{t^2 n}{\sigma^2}\right)
        \right).
\end{equation}
\end{proposition}

\begin{remark}
    For simplicity, we may assume that
    $1 \leq c_1$ when we apply \cref{prop:fn_modified}.
\end{remark}

\paragraph{Confidence regimes.}
Directly rearranging \eqref{eq:fn_modified} from a tail bound to a confidence 
interval bound requires to solve a
transcendental equation which does not admit a simple closed form solution. 
However, we can still derive an upper bound on the confidence intervals
that reflects the superposition of polynomial and sub-gaussian tail in 
\eqref{eq:fn_modified}.
By introducing 
\begin{equation}
\label{eq:fn_confidence_derivation1}
  \delta := 
  2 \max\left\{
   \frac{c_1 Q}{t^q n^{q-1}},
    c_1 \exp\left(-c_2 \frac{t^2 n}{\sigma^2}\right)
  \right\},
\end{equation}
we have $\P[ n^{-1}\Norm{ S_n }_\X  \leq t] \geq 1- \delta$ by \eqref{eq:fn_modified}.
Rearranging \eqref{eq:fn_confidence_derivation1}, we have
\begin{equation}
    \label{eq:fn_confidence_derivation2}
    t \geq \left(\frac{2 c_1 Q}{\delta n^{q-1}} \right)^{1/q}
    \quad
    \text{and}
    \quad
    t \geq
    \sigma \sqrt{ \frac{\log(2c_1/\delta)}{c_2 n} },
\end{equation}
immediately leading to the following confidence bound.

\begin{corollary}[Confidence bound]
\label{cor:confidence_bound}
Under the assumptions of \cref{prop:fn_modified}, for all $\delta \in (0,1)$, we have 
\begin{equation}
    \label{eq:fn_confidence}
    \Norm{ \frac{1}{n} S_n }_\X 
    \leq
    \max
    \left\{
    \left(\frac{2 c_1 Q}{\delta n^{q-1}} \right)^{1/q}  ,
    \sigma \sqrt{ \frac{\log(2c_1/\delta)}{c_2 n} }
    \right\}
\end{equation}
with probability at least $1 - \delta$.
\end{corollary}

For every fixed $\delta$ and $n \to \infty$, this
shows the typical subgaussian behavior
and a convergence rate of $\frac{1}{n} S_n$ of the order $n^{-1/2}$
with high probability.
Interpreting the right hand side as a function
of $\delta$ for a fixed sample size $n$ however, 
the above bound characterizes a \emph{confidence regime change}
at $\bar \delta(n)$, which we define as the solution to the equation
\begin{equation}
    \label{eq:regime_change_equation}
    \left(\frac{2 c_1 Q}{\bar\delta(n)\, n^{q-1}} \right)^{1/q}  
    =
    \sigma \sqrt{ \frac{\log(2c_1/\bar\delta(n))}{c_2 n}}.
\end{equation}
In fact, in the \emph{polynomial confidence regime}
$\delta < \bar\delta(n)$, the dependence of the upper bound given
in \cref{cor:confidence_bound} on $\delta$ is clearly
worse than in the \emph{subgaussian regime} $\delta \geq \bar\delta(n)$
which is characterized by a logarithmic dependence on $\delta$.
In contrast, the polynomial confidence regime allows for a better
sample dependence of $n^{-(q-1)/q}$.



\paragraph{Sharpness of the tail bound.}

Both the subgaussian term and
the polynomial term in the right hand side
of the bound given by \cref{prop:fn_modified}
can generally not be improved without additional assumptions.
We repeat a similar argument as
the one given in \citep[Proposition 9]{LouEtAl2022},
which is given in the context of linear processes.
Let $\xi, \xi_1, \dots, \xi_n$ be independent real-valued random variables
drawn from a centered $t$-distribution with $q$ degrees of freedom, i.e.\
$\E[\xi^{q-c}] < \infty$ for all $0 < c \leq q$ and let 
$\sigma^2 := \E[\xi^2] = q / (q-2)$.
Then \citep[Theorem 1.9]{Nagaev1979}
shows that we have
\begin{equation*}
    \P\left[ S_n / n > t \right]
    =
    \P\left[ S_n / \sigma > nt / \sigma \right]
    \sim 1 - \phi(n^{1/2} t / \sigma ) + n( 1- F_{\sigma^{-1}\xi}(nt/\sigma) )
    \quad \text{ as } n \to \infty
\end{equation*}
for $nt/\sigma \geq n^{1/2}$,
where $\Phi$ is the standard normal cumulative distribution function
and $F_{\sigma^{-1}\xi}$ is the cumulative distribution function of
$\sigma^{-1}\xi$.
We now note that we have the basic property
$F_{\sigma^{-1}\xi}(nt/\sigma) = F_\xi(nt)$
and we can show that the distribution of $\xi$ satisfies
$1 - F(nt) \sim C_q / (nt)^q$
as $nt \to \infty$, where $C_q$ is
a constant depending exclusively on $q$.
In total, we obtain
\begin{equation*}
    \P\left[ S_n / n > t \right]
    \sim 1 - \Phi(n^{1/2} t / \sigma ) + \frac{C_q}{t^q n^{q-1} }
    \quad \text{ as } n \to \infty
\end{equation*}
for $nt/\sigma \geq n^{1/2}$, showing
that \cref{prop:fn_modified} is asymptotically optimal.

\section*{Acknowledgements}
Dimitri Meunier and Arthur Gretton are supported by
the Gatsby Charitable Foundation.
This work was partially funded by the German Federal Ministry of Research, Technology and Space as part of the 6G-CAMPUS project (grant no. 16KISK194).
\yurinsky{
The authors express their gratitude to Christian Fiedler
for pointing out inconsistensies in the proof of
\cite[Lemma 3.5.1]{Yurinsky1995} and
for his contribution to the corresponding corrections
in \cref{sec:exponential_moment_bound} and \cref{sec:chernoff_optimization}.
}

\bibliographystyle{unsrtnat}
\bibliography{references, references-heavy-tailed-noise}
\addcontentsline{toc}{section}{References}

\newpage

\appendix

\part*{Appendices}

The appendices are organized as follows.

\camera{
\cref{sec:experiment} contains a numerical experiment 
 which confirms the behavior of the excess risk
as described by the bounds provided in the main text in a very basic
setting.}
In \cref{app:proofs}, we report proofs of the results in
the main text of this paper.
In \cref{app:polynomial_regime}, we provide an additional
excess risk bound for the polynomial confidence
regime based on \cref{prop:excess_risk_bound}.
We collect some tail bounds 
and concentration results in \cref{app:concentration_bounds}.
In particular, \cref{app:fuk-nagaev} contains a proof of
\cref{prop:fn_modified}.
In \cref{app:miscellaneous}, we recall additional miscellaneous
inequalities required for our proofs.

\camera{
\section{Numerical experiment}
\label{sec:experiment}

We provide a basic numerical experiment\footnote{
We provide the source code on GitHub:
\url{https://github.com/mollenhauerm/krr-heavy-tailed}.
The experiment can be run on the CPU of any consumer laptop.
} showing that
the confidence regimes are not only occurring in the upper bound
on the excess risk
given by \cref{prop:excess_risk_bound}, but are effectively reflected in the
distribution of the excess risk in practice, even for very simple models.

We consider the input space $\X:= \R$ equipped with the RKHS $\mathcal{H}$ induced by the radial basis kernel 
$k(x_1, x_2) :=  \exp( \tfrac{ - |x_1 - x_2|^2 }{2} )$
and define the target function 
$f_\star(x) := \sum_{i=1}^5 a_i k(x_i, x) \in \mathcal{H}$, $x \in \X$
for vectors $\mathbf{a} := (2, -1, -3 ,1,2)$
and $\mathbf{x} := (-4,-2-0,3,7)$. 
We define the covariate distribution
$X \sim \pi := \mathcal{N}(0, 1)$ on $\X$ and
generate independent observation pairs with
a light-tailed noise distribution and
a heavy-tailed distribution with identical variance based on
\begin{enumerate}[label=(\roman*)]


    \item the \emph{light-tailed noise model} given by
    \begin{equation}
        \label{eq:experiment_light}
        Y_{(\mathcal{N})} = f_\star(X) + \epsilon_{(\mathcal{N})},
    \end{equation}
    where the noise $\epsilon_{(\mathcal{N})} \sim \mathcal{N}(0, \sigma^2)$ follows a centered Gaussian
    distribution with variance $\sigma^2 = 3$.
    
    \item the \emph{heavy-tailed noise model} with a finite number of higher moments
    given by
    \begin{equation}
        \label{eq:experiment_heavy}
        Y_{(t)} = f_\star(X) + \epsilon_{(t)}, 
    \end{equation}
    where the noise $\epsilon_{(t)} \sim \mathbf{t}(0, \nu)$ follows a
    centered $t$-distribution with $\nu = 3$ degrees of freedom
    and $\E[\epsilon_{(t)}^2] = 3$.
    \end{enumerate}

\begin{figure}[H]
    \centering
    \begin{subfigure}{0.6\textwidth}
        \includegraphics[width=\textwidth]{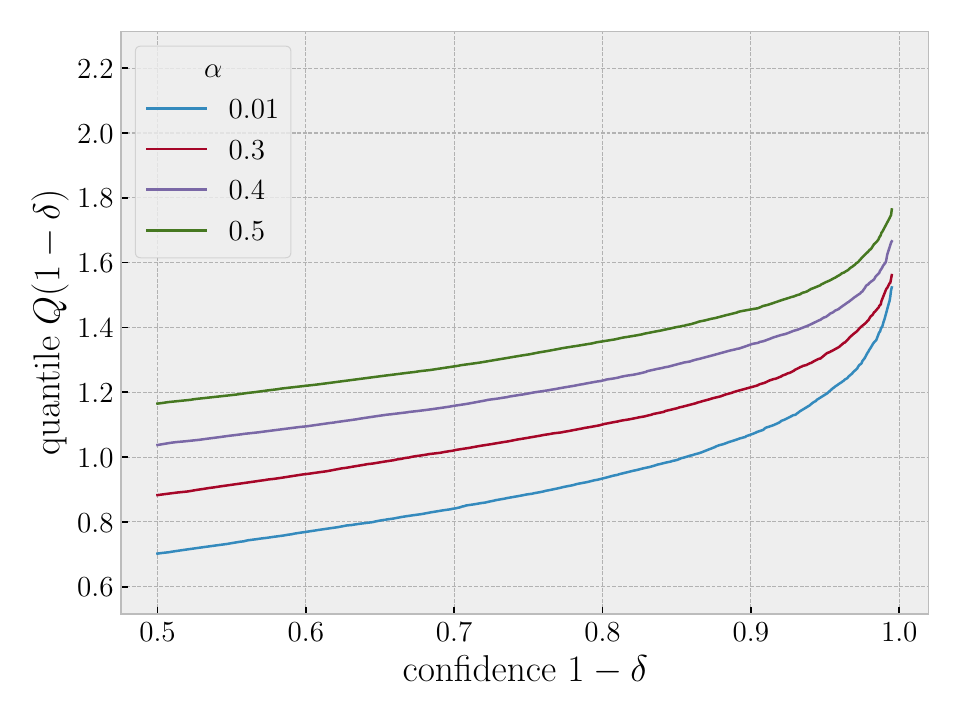}
        \caption{Light-tailed noise}
    \end{subfigure}
    \begin{subfigure}{0.6\textwidth}
        \includegraphics[width=\textwidth]{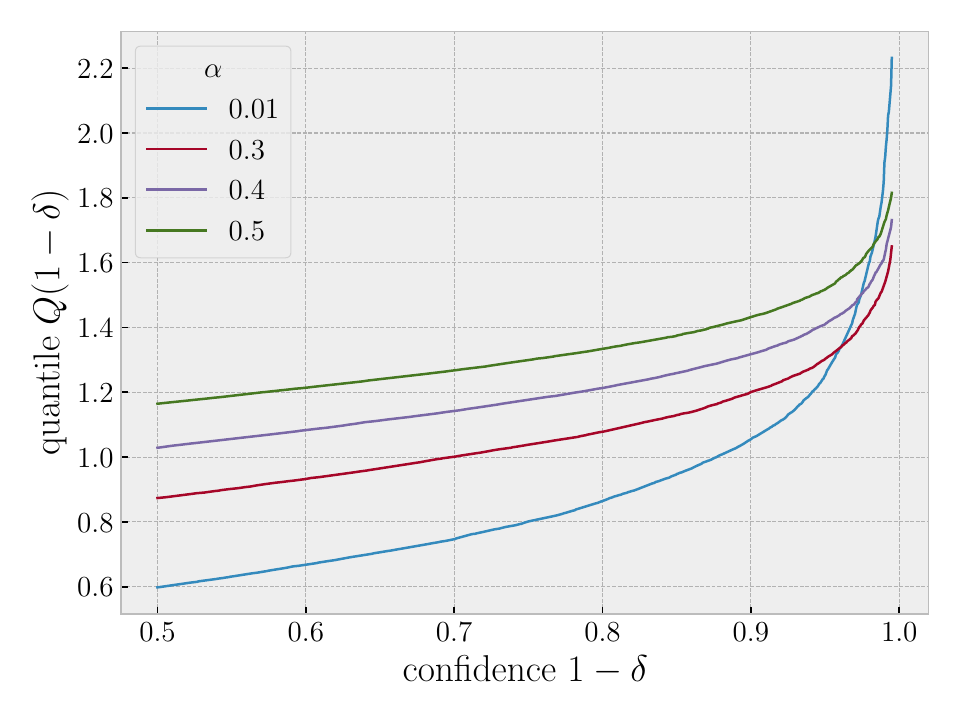}
        \caption{Heavy-tailed noise}
        \label{fig:subfig1}
    \end{subfigure}
    \caption{Empirical approximations of the quantile function of the excess risk 
    $ \norm{\inc\widehat f_\alpha - f_\star}_{L^2(\pi)}$ for
    (a) the light-tailed noise model given in \eqref{eq:experiment_light} and
    (b) in the $t$-distributed noise model given by \eqref{eq:experiment_heavy} 
    for different choices of the regularization
    parameter $\alpha$ and fixed sample size $n=20$.}
     \label{fig:experiment1}
\end{figure}

For both models, 
we compute $\widehat{f}_\alpha$ based on the generated sample pairs
for a sample size $n=20$
and record the error $ \norm{\inc\widehat f_\alpha - f_\star}_{L^2(\pi)}$, 
which we approximate through
Monte Carlo simulation by drawing samples from $\pi$.
We perform the above computation for a selection of
different regularization parameters $\alpha$,
repeating each experiment across $10000$ random seeds (per model and choice of $\alpha$).
Based on the recorded errors, we
empirically approximate the quantile function of the distributions of
$ \norm{\inc\widehat f_\alpha - f_\star}_{L^2(\pi)}$ given by
\begin{equation*}
    Q(1-\delta) := \inf\left\{ 
    t \colon \P\left[ \norm{\inc\widehat f_\alpha - f_\star}_{L^2(\pi)}  \leq t \right] \geq 1-\delta 
    \right\}, \quad \delta \in [0,1]
\end{equation*}
The resulting 
approximated quantile functions are visualized in \cref{fig:experiment1}
for all choices of $\alpha$ and both models.

While the excess risk quantiles exhibit the same qualitative behaviour and are of
the same order in both models for moderate confidence levels, the excess risk quantiles 
of the heavy-tailed model increase rapidly for a small regularization parameter
beyond a high confidence level of $1 - \delta \approx 0.95$.
Consequently, the heavy-tailed model requires a stronger regularization
than the light-tailed model in high confidence settings, as explained by the 
transition of the excess risk from a subgaussian
confidence regime to a polynomial confidence regime in \cref{sec:ridgre_regression}.
}

\section{Proofs}
\label{app:proofs}

We provide the proofs for the results
presented in the main text of this work.

\subsection{Proof of \cref{prop:excess_risk_bound}}
\label{proof:excess_risk_bound}

We investigate the classical \emph{bias-variance decomposition}
\begin{equation*}
    I_\pi\widehat f_\alpha  - f_\star 
    =
    I_\pi\widehat f_\alpha - I_\pi f_\alpha 
    +
    I_\pi f_\alpha - f_\star
\end{equation*}
and bound both terms individually.



\subsubsection{Bounding the variance $\inc \widehat f_\alpha  - \inc f_\alpha$.}

We first collect some generic properties of $f_\alpha$. The identity $f_\alpha = \regcov \inc^* f_\star$
gives
\begin{equation}
    \label{eq:alpha_f_alpha}
    \mbe[ \empinc ^*( f_\star - \empinc f_\alpha ) ] = \inc ^* (f_\star - \inc f_\alpha) = \alpha f_\alpha.
\end{equation}

Based on \eqref{eq:solution_rkhs}, \eqref{eq:empirical_solution} and
\eqref{eq:empirical_rhs}, we decompose the variance in $\mathcal{H}$ as
\begin{align*}
    \widehat f_\alpha - f_\alpha &=
    \empregcov
    \left\{ 
    \empinc^* ( f_\star - \empinc f_\alpha )
    - \alpha f_\alpha
    +  \frac{1}{n}\sum_{i=1}^n \phi(X_i) \epsilon_i 
    \right\} \\
    &= 
    \empregcov
    \left\{ 
    \empinc^* ( f_\star - \empinc f_\alpha )
    - \inc^* (f_\star - \inc f_\alpha) 
    +  \frac{1}{n}\sum_{i=1}^n \phi(X_i) \epsilon_i 
    \right\} && \text{by \eqref{eq:alpha_f_alpha}} \\
    &=
    \empregcov
    \left\{ 
    \Delta_1 + \Delta_2 
    \right\}, 
\end{align*}
where we introduce the $\mathcal{H}$-valued random variables
\begin{equation}
    \label{eq:Deltas}
\Delta_1 = \empinc^* ( f_\star - \empinc f_\alpha )
    -  \mbe[ \empinc ^*( f_\star - \empinc f_\alpha ) ], \qquad  
\Delta_2 = \frac{1}{n}\sum_{i=1}^n \phi(X_i) \epsilon_i \;.
\end{equation}
We now obtain an $L^2(\pi)$-norm bound on the variance
as
\begin{align}
    \norm{\inc ( \widehat f_\alpha - f_\alpha)}_{L^2(\pi)}
    &=
    \norm{ \cov^{1/2} 
    ( \widehat f_\alpha - f_\alpha)}_{\mathcal{H}}
    =
    \norm{ \cov^{1/2} 
    \empregcov
    \left\{ 
    \Delta_1 + \Delta_2 
    \right\}  }_{\mathcal{H}} \nonumber \\
    &\leq \sqrt{2}\; \cB(n, \delta, \alpha)^{1/2}
        \Norm{ 
            (\widehat{\mathcal{C}}_\pi + \alpha \idop_{\mathcal{H}})^{1/2} 
        \empregcov
        \left\{ 
        \Delta_1 + \Delta_2 
        \right\} 
        }_{\mathcal{H}} \nonumber \\
    &\leq
    2 \sqrt{2}
    \Norm{ 
            (\widehat{\mathcal{C}}_\pi + \alpha \idop_{\mathcal{H}})^{-1/2} 
        \left\{ 
        \Delta_1 + \Delta_2 
        \right\} 
        }_{\mathcal{H}} \nonumber \\
    &\leq
    2\sqrt{\frac{2}{\alpha}}( \norm{\Delta_1}_\mathcal{H} 
        + \norm{ \Delta_2 }_\mathcal{H}) \label{eq:l2_bound},
\end{align}
with probability at least $1-\delta$, 
where $\cB(n, \delta, \alpha)$ is defined in 
\eqref{eq:simplify-inverse-conc} and
where we apply  \cref{lem:concentration_weighted_norm} in the first inequality, \cref{cor:concentration_inverse_product} in the second inequality, and additionally 
the fact that we have
$\norm{ (\widehat{\mathcal{C}}_\pi + \alpha \idop_{\mathcal{H}})^{-1/2}}_{\bounded(\mathcal{H})} \leq \alpha^{-1/2} $ in the last inequality.

\vspace{0.2cm}

We now provide individual
confidence bounds for $\Delta_1$ and $\Delta_2$.


\paragraph{Bounding $\Delta_1$: Bennett inequality.}

We introduce the independent and identically distributed random variables
$\xi_i := ( f_\star(X_i) - f_\alpha(X_i) ) \phi(X_i) \in \mathcal{H}$
and note that $\Delta_1 := \frac{1}{n} \sum_{i=1}^n (\xi_i - \E[\xi_i])$
and proceed similarly as in the proof of \citep[Theorem 1 \& Theorem 5]{Smale2007}.

To apply Bennett's inequality, we need to bound the norm of the $\xi_i$. 
First, note that almost surely, we have by \eqref{eq:src}, for some $f \in L^2(\pi)$, such that $\norm{f}_{L^2(\pi)} \leq R$
\begin{align*}
\norm{f_\alpha (X_i) \phi(X_i)}_{\cH} &= \norm{\langle f_\alpha , \phi (X_i)\rangle_{\cH} \phi(X_i)}_{\cH} \\
&\leq \kappa^2 \Vert f_\alpha \Vert _\cH \\
&= \kappa^2 \Vert  (\cC_\pi + \alpha \idop_{\cH})^{-1}\inc^* f_\star \Vert _\cH\\
&= \kappa^2 \Vert  \inc^*(\cT_\pi + \alpha \idop_{\cH})^{-1} \cT_\pi^\nu f \Vert _\cH.
\end{align*}


A short calculation shows that for $\nu \geq 1/2$,
the map $h(t)=t^{\nu+1/2}(t+\alpha)^{-1}$ is increasing on the spectrum of $\cT_\pi$ with 
\[ \sup_{t \in (0, \kappa^2]} | h(t)| =   \sup_{t \in (0, \kappa^2]} t^{\nu - 1/2} \frac{t}{t+\alpha} \leq \kappa^{2(\nu - 1/2)},\] 
for all $\alpha > 0$. Hence, 
\[ \Vert  \inc^*(\cT_\pi + \alpha \idop_{\cH})^{-1} \cT_\pi^\nu f \Vert _\cH \leq R\cdot \sup_{t \in (0, \kappa^2]} |h(t)| 
\leq R\cdot  \kappa^{2(\nu - 1/2)}  \]
and therefore
\[ \norm{f_\alpha (X_i) \phi(X_i)}_{\cH} \leq R\cdot  \kappa^{2\nu + 1}  \;.  \]
This gives 
\begin{align*}
    \norm{\xi_i}_\mathcal{H} 
    &\leq
    \kappa  \norm{f_\star}_{L^\infty(\pi)} + \norm{f_\alpha (X_i) \phi(X_i)}_{\cH} \\
    &\leq
    \kappa \norm{f_\star}_{L^\infty(\pi)} +  R\cdot  \kappa^{2\nu + 1} .
\end{align*}
Furthermore, by \citep[Proposition 3.3]{DeVitoEtAl06} and \cite[Lemma 25]{fischer2020sobolev} we see
\begin{align*}
    \E[ \norm{\xi_i}^2_\mathcal{H} ]
    &\leq \kappa^2 \norm{f_\star - I_\pi f_\alpha}^2_{L^2(\pi)} \\
   &=\kappa^2 \norm{( \idop_{\cH} - \cT_\pi (\cT_\pi + \alpha \idop_{\cH})^{-1}) \cT_\pi^\nu f }^2_{L^2(\pi)} \\
&\leq \kappa^2 R^2 \alpha^{2\min\{\nu , 1\}}.
\end{align*}

Applying \cref{prop:bennett} to $\Delta_1 = \frac{1}{n} 
\sum_{i=1}^n (\xi_i - \E[\xi_i])$,
with probability at least $1-\delta$ for for any $1 > \delta > 0$,
we have
\begin{align}
\label{eq:delta_1_bound}
    \norm{\Delta_1}_\mathcal{H}
    &\leq
    \frac{2  (\kappa \norm{f_\star}_{L^\infty(\pi)} +  R\cdot  \kappa^{2\nu + 1} )\log(2/\delta)}{n}
    +
    \sqrt{\frac{2  \kappa^2 R^2 \alpha^{2\min\{\nu , 1\}}  \log(2/\delta)}{n}} \nonumber \\
    &\leq \tilde  C_{\diamond} \cdot \left(  \frac{\log(2/\delta)}{n} + 
  \sqrt{ \frac{\alpha^{2 \min\{\nu , 1\}} \log(2/\delta)}{n} } \right) ,
\end{align}
where we use $\sqrt{2 \log(2/\delta)} \leq 2 \log(2/\delta)$ and with 
\begin{equation}
\label{eq:diamond}
 \tilde C_\diamond = 2  \kappa \cdot \max\{ \norm{f_\star}_{L^\infty(\pi)} +  R\cdot  \kappa^{2\nu }  ,  R\}. 
\end{equation}

\paragraph{Bounding $\Delta_2$: Fuk--Nagaev inequality.}

We introduce the independent and identically distributed random variables
$\zeta_i := \phi(X_i) \epsilon_i$ and note that 
$\Delta_2 = \frac{1}{n} \sum_{i=1}^n \zeta_i$.
By \cref{assump:mom}, 
we clearly have $\E[\zeta_i] = \E\left[ \phi(X_i) \E[\epsilon_i \mid X] \right] = 0$
as well as
\begin{equation*}
    \E[ \norm{\zeta_i}^2_\mathcal{H}]
    < \kappa^2 \sigma^2
    \quad \text{and} \quad
    \E[ \norm{\zeta_i}^q_\mathcal{H}]
    < \kappa^q Q.
\end{equation*}
We can therefore apply \cref{cor:confidence_bound}
to $\Delta_2 = \frac{1}{n} \sum_{i=1}^n \zeta_i$
and obtain
\begin{equation}
    \label{eq:delta_2_bound}
    \Norm{ \Delta_2 }_\mathcal{H} 
    \leq
    \max
    \left\{
    \left(\frac{2 c_1 \kappa^q Q}{\delta n^{q-1}} \right)^{1/q}  ,
    \kappa\sigma \; \sqrt{\frac{\log(2c_1/\delta)}{ c_2 n} }
    \right\}
\end{equation}
with probability at least $1-\delta$.

\paragraph{Collecting the terms.} By the union bound, the three individual 
confidence bounds \eqref{eq:l2_bound}, \eqref{eq:delta_1_bound} and \eqref{eq:delta_2_bound} and  hold simultaneously
with probability at least $1-\delta$. 
We have
\begin{align*}
    \norm{\inc ( \widehat f_\alpha - f_\alpha)}_{L^2(\pi)}
    &\leq \frac{C_\diamond}{\sqrt \alpha} 
    \left(    \frac{\log(6/\delta)}{n} + 
  \sqrt{ \frac{\alpha^{2 \min\{\nu , 1\}} \log(6/\delta)}{n} }  
    +    \eta(\delta, n)
    \right) .
\end{align*}
Here, 
we set  
\begin{equation*}
    \eta(\delta, n) :=
    \max
    \left\{
    \left(\frac{Q}{\delta n^{q-1}} \right)^{1/q}  ,
      \sigma\; \sqrt{\frac{ \log(6c_1/\delta)}{n }}
    \right\}
\end{equation*}
and 
\begin{equation}
\label{eq:diamond}
C_\diamond = 2\sqrt{ 2 } \max\{\tilde  C_{\diamond} ,   \kappa \cdot \max\{(6c_1)^{1/q}, 1/\sqrt{c_2} \}\}.
\end{equation}


\subsubsection{Bounding the bias $\inc f_\alpha  - f_\star$} 
\label{sec:bounding-bias}

We apply the standard theory of source conditions in 
inverse problems as discussed for example by \citep{EHN96}.
In particular, under the $L^2(\pi)$-source condition
given by \eqref{eq:src}, we have
\begin{align*}
    \norm{ \inc f_\alpha   - f_\star}_{L^2(\pi)}
    \leq R \alpha^{\min\{ \nu, 1 \}},
\end{align*}
see e.g.\ \citep[Proposition 3.3]{DeVitoEtAl06}.
Combining the variance bound and the bias bound yields the result.

\subsection{Proof of \cref{cor:low_confidence}}
\label{proof:low_confidence}

Let $\delta \in D_1(n,q)$. Using that $\alpha \leq \kappa^2$ and recalling that $c_1 \geq 1$, we find by \cref{prop:excess_risk_bound} for any 
$\delta \in D_1(n, q)$ with confidence $1-\delta$
\begin{align*}
    \norm{\inc ( \widehat f_\alpha - f_\star)}_{L^2(\pi)}
    &\leq  R \alpha^{\min\{ \nu, 1 \}} +  \frac{C_\diamond}{\sqrt \alpha} 
    \left(    \frac{\log(6/\delta)}{n} + 
  \sqrt{ \frac{\alpha^{2 \min\{\nu , 1\}} \log(6/\delta)}{n} }  
    +     \sigma \sqrt{  \frac{\log(6c_1/\delta)}{ n} }
    \right) . 
\end{align*}
Assuming 
\begin{equation}
\label{eq:alpha1}
\alpha^{ \min\{\nu , 1\}} \sqrt{n}  \geq \sqrt{\log(6/\delta)} 
\end{equation}
gives 
\[ \frac{\log(6/\delta)}{n} \leq  \sqrt{ \frac{\alpha^{2 \min\{\nu , 1\}} \log(6/\delta)}{n} }  
                            \leq \kappa^{ 2\min\{\nu , 1\}}  \sqrt{ \frac{\log(6c_1/\delta)}{n} }.
\]
Hence, 
\begin{align*}
    \norm{\inc  \widehat f_\alpha - f_\star}_{L^2(\pi)}
    &\leq   R \alpha^{\min\{ \nu, 1 \}} +  c_3 \sqrt{  \frac{\log(6c_1/\delta)}{ \alpha n} },  
\end{align*}
with $c_3 = 3C_\diamond \max\{ \kappa^{2 \min\{\nu , 1\}}, \sigma\} $. 
Ensuring that the remaining estimation error contribution will not be dominated by the approximation error, we see that the regularization parameter has to satisfy  
\begin{equation}
\label{eq:alpha1b}
\alpha \geq 
c_4\;\left( \frac{\log(6c_1/\delta)}{n} \right)^{\frac{1}{2\min\{\nu ,1\} +1}} ,
\end{equation}
with $c_4=   (c_3  /R)^{\frac{1}{\min\{\nu, 1\} +1/2 }} $. 
We therefore obtain
\begin{align*}
    \norm{\inc  \widehat f_\alpha - f_\star}_{L^2(\pi)}
    &\leq \tilde{c}_1 R \;\left( \frac{\log(6c_1/\delta)}{n} \right)^{\frac{\min\{ \nu, 1 \}}{2\min\{\nu ,1\} +1}} , 
\end{align*}
with confidence $1-\delta$ 
and 
$\tilde{c}_1 = 2c_4^{\min\{ \nu, 1 \}}$. In order for this bound to hold, the regularization strength has to fulfill condition \eqref{eq:alpha-coice}, \eqref{eq:alpha1} and \eqref{eq:alpha1b}, that is,
\begin{equation}
    \label{eq:alpha_1_restriction}
\alpha \geq \max\left\{ c_4 \left( \frac{\log(6c_1/\delta)}{n} \right)^{\frac{1}{2\min\{\nu ,1\} +1}}  , 
\left( \frac{\log(6/\delta)}{n} \right)^{\frac{1}{2\min\{ \nu, 1 \}}} , 
    C_\kappa \frac{\log(6/\delta)}{\sqrt n}\right\}
\end{equation}   
in addition to the restriction $\alpha \leq \kappa^2$.
As we always have $n^{-\frac{1}{2\min\{\nu ,1\}}} \leq n^{-1/2} \leq n^{-\frac{1}{2\min\{\nu ,1\} +1}} $ for $\nu \geq 1/2$
and $\log(6c_1/\delta) \geq \log(6 / \delta)$,
there exists a constant $\tilde{c}_2$ such that
\begin{equation*}
    \alpha_1(n, \delta) := 
    \tilde{c}_2 \left( \frac{\log(6 c_1 / \delta)}{n} \right)^{\frac{1}{2\min\{\nu ,1\} +1}}
\end{equation*}
satisfies \eqref{eq:alpha_1_restriction}.

\subsection{Proof of \cref{prop:excess_risk_capacity_new}}
\label{proof:excess_risk_capacity}

We split again into \emph{variance} and \emph{bias}: 
\begin{equation*}
    I_\pi\widehat f_\alpha  - f_\star 
    =
    I_\pi (\widehat f_\alpha -  f_\alpha) 
    +
    I_\pi f_\alpha - f_\star .
\end{equation*}


\subsubsection{Bounding the variance $I_\pi \widehat f_\alpha - I_\pi  f_\alpha $}

Following \citep[Appendix C]{mucke2019beating}, we further decompose the variance term in $L^2(\pi)$ as 

\begin{align}
\label{eq:variance}
I_\pi \widehat f_\alpha - I_\pi  f_\alpha 
&=  I_\pi \widehat f_\alpha - \inc \empcov ( \empcov + \alpha \idop_\cH )^{-1} f_\alpha + \inc \empcov ( \empcov + \alpha \idop_\cH )^{-1} f_\alpha - I_\pi  f_\alpha \nonumber \\
&=  \underbrace{\inc (\empcov + \alpha \idop_\cH )^{-1} \left( \widehat{\Upsilon} - \empcov f_\alpha  \right)}_{=: I_1}  
+  \underbrace{\alpha \inc ( \empcov + \alpha \idop_\cH )^{-1} f_\alpha }_{=:I_2},
\end{align}
where we use  the definition of $\widehat f_\alpha$ in \eqref{eq:empirical_solution} for the first term and the identity 
\[ \empcov (\empcov + \alpha \idop_\cH )^{-1} - \idop_\cH = \alpha (\empcov + \alpha \idop_\cH)^{-1}   \]
for the second summand.

\subsubsection{Bounding the first term $I_1$} 

We introduce the shorthand notation
$\cC_\alpha:= \cov + \alpha \idop_\cH  $, $ \widehat{\cC}_\alpha := \empcov + \alpha \idop_\cH$.
We note that we may write  
\begin{align*}
 \cov ^{1/2} (\empcov + \alpha \idop_\cH )^{-1} 
 &=  \cov ^{1/2} \cC_\alpha^{-1/2}
 \cC_\alpha^{1/2} \widehat{\cC}_\alpha^{-1/2} 
 \widehat{\cC}_\alpha^{-1/2}   \cC_\alpha^{1/2}\cC_\alpha^{-1/2}. 
\end{align*} 
Recall that $\Vert   \cov ^{1/2} \cC_\alpha^{-1/2} \Vert _{\cL(\cH)} \leq 1$ and that by \cref{cor:concentration_inverse_product_under_evd}
and \cref{lem:cordes}, with confidence $1-\delta$, we have 
\[  \Vert  \cC_\alpha^{1/2} \widehat{\cC}_\alpha^{-1/2}  \Vert _{\cL(\cH)} = 
\Vert  \widehat{\cC}_\alpha^{-1/2}   \cC_\alpha^{1/2} \Vert _{\cL(\cH)}  \leq 2,\]
provided \eqref{eq:ass:eff:dim_new} is satisfied. 
Hence, by \eqref{eq:polar_decomposition}, with confidence $1-\delta$
\begin{align}
\label{eq:I1bound1}
\Vert I_1\Vert _{L^2(\pi)} = \norm{\inc (\empcov + \alpha \idop_\cH )^{-1} ( \widehat{\Upsilon} - \empcov f_\alpha  )}_{L^2(\pi)}
&\leq    4 \; \norm{ (\cov + \alpha \idop_\cH )^{-1/2}( \widehat{\Upsilon} - \empcov f_\alpha )}_\cH .
\end{align}
We proceed by further splitting 
\begin{align*}
 (\cov + \alpha \idop_\cH )^{-1/2}(\widehat{\Upsilon} - \empcov f_\alpha )  &=  (\cov + \alpha \idop_\cH )^{-1/2} [(\widehat{\Upsilon} -  \empcov f_\alpha) 
- (\inc^* f_\star - \cov f_\alpha) ] \\
& \quad  +    (\cov + \alpha \idop_\cH )^{-1/2} (\inc^* f_\star - \cov f_\alpha).
\end{align*}
Note that by \eqref{eq:empirical_rhs}, we have
\[  \widehat{\Upsilon}    =     \empinc^* f_\star
    + \frac{1}{n} \sum_{i=1}^n \phi(X_i) \epsilon_i . \]
Introducing 
\[  \Delta_ 1:= ( \empinc^* f_\star -  \empcov f_\alpha) 
- (\inc^* f_\star - \cov f_\alpha) ,  \qquad \Delta_2:= \frac{1}{n} \sum_{i=1}^n \phi(X_i) \epsilon_i, \]
we obtain 
\begin{align*}
 (\cov + \alpha \idop_\cH )^{-1/2}(\widehat{\Upsilon} - \empcov f_\alpha )  
 &=  I_{11} + I_{12} + I_{13},
\end{align*}
where we set 
\begin{align*}
I_{11}&:=   (\cov + \alpha \idop_\cH )^{-1/2} \Delta_ 1 \\
I_{12}&:=  (\cov + \alpha \idop_\cH )^{-1/2} \Delta_ 2\\
I_{13}&:=    (\cov + \alpha \idop_\cH )^{-1/2} (\inc^* f_\star - \cov f_\alpha) .
\end{align*}
Thus, with \eqref{eq:I1bound1}, we have
\begin{equation}
\label{eq:3norms}
\Vert I_1\Vert _{L^2(\pi)} 
\leq  4 \left( \Vert I_{11}\Vert _\cH  + \Vert I_{12}\Vert _\cH  + \Vert I_{13}\Vert _\cH \right)  
\end{equation}
We proceed to bound the terms on the right hand side individually.

\paragraph{Bounding $\Vert I_{11}\Vert _\cH$.} To bound the norm of $I_{11}$, we apply Bennett's inequality \cref{prop:bennett} to the independent and identically distributed random  variables 
\[  \xi_i:= (\cov + \alpha \idop_\cH )^{-1/2} ( f_\star(X_i) -  f_\alpha (X_i)) \phi(X_i). \]
Then
\[ I_{11} = \frac{1}{n}\sum_{i=1}^n  \xi_i - \mbe[ \xi_i] .  \]
Furthermore, 
\begin{align*}
\mbe[\Vert \xi_i\Vert ^2_\cH] &\leq \sup_{X \in \cX}|f_\star(X) -  f_\alpha (X)|^2 \;  \mbe[ \Vert  (\cov + \alpha \idop_\cH )^{-1/2}\phi(X_i) \Vert ^2_\cH ] .
\end{align*}
since $f_\alpha \in \cH$ and by \eqref{eq:src}, we have 
for $\nu  \geq 1/2$
\begin{align}
\label{eq:falpha}
 |f_\alpha (X)| &= |\langle f_\alpha , \phi(X) \rangle_{\cH}| \nonumber\\
&= |\langle (\cov + \alpha \idop_\cH )^{-1} \inc^* \cT_\pi^\nu f, \phi(X) \rangle_{\cH}| \nonumber\\
&= |\langle  f,  \cT_\pi^\nu \inc (\cov + \alpha \idop_\cH )^{-1} \phi(X) \rangle_{L^2(\pi)}| \nonumber\\
&\leq \kappa R \Vert \inc^* (\cT_\pi + \alpha \idop_{L^2(\pi)} )^{-1}  \cT_\pi^\nu \Vert _{\cL(L^2(\pi), \mathcal{H})} \nonumber\\
&\leq \kappa R \sup_{0 < t \leq \kappa^2} | t^{\nu +1/2}(t + \alpha)^{-1} |\nonumber \\
&\leq \kappa^{2\nu} R .
\end{align}
We proceed with writing  
\begin{align*}
 \mbe[ \Vert  (\cov + \alpha \idop_\cH )^{-1/2}\phi(X_i) \Vert ^2_\cH ] &=  \mbe[\tr((\cov + \alpha \idop_\cH )^{-1}\phi(X_i)\otimes \phi(X_i) )] = \cN(\alpha) . 
\end{align*}
Hence, 
\begin{align*}
\mbe[\Vert \xi_i\Vert ^2_\cH] &\leq 2( \Vert  f_\star\Vert ^2_\infty + \kappa^{4\nu} R^2  )\cN(\alpha) =: \tilde \sigma^2.
\end{align*}
Moreover, 
\begin{align*}
\Vert \xi_i\Vert _\cH &= \Vert (\cov + \alpha \idop_\cH )^{-1/2} ( f_\star(X_i) -  f_\alpha (X_i)) \phi(X_i)\Vert _\cH \\
&\leq \Vert f_\star(X_i)(\cov + \alpha \idop_\cH )^{-1/2} \phi(X_i)\Vert _\cH 
+ \Vert  f_\alpha (X_i)(\cov + \alpha \idop_\cH )^{-1/2}\phi(X_i)\Vert _\cH \\
&\leq \frac{\kappa \Vert f_\star\Vert _\infty}{\sqrt{\alpha}} + \frac{\kappa f_\alpha (X_i)}{\sqrt{\alpha}}\\
&\leq  \frac{\kappa }{\sqrt{\alpha}}\;  (\Vert f_\star\Vert _\infty + \kappa^{2\nu} R) =: L . 
\end{align*}
In the last step we use \eqref{eq:falpha}. From \cref{prop:bennett}, we obtain with confidence 
$1-\delta$
\begin{align}
\label{eq:I11}
\Vert I_{11}\Vert _\cH 
    &\leq \frac{2 L \log(2/\delta)}{n}  +
            \sqrt{\frac{2 \tilde \sigma^2 \log(2/\delta)}{n}} \nonumber \\
    &\leq c_{11}\left( \frac{ \log(2/\delta)}{\sqrt{\alpha} n}  + \sqrt{\frac{\cN(\alpha) \log(2/\delta)}{n}}\right),
\end{align}
with $c_{11} = 2(\Vert f_\star\Vert _\infty + \kappa^{2\nu} R)\; \max\{1, \kappa \} $.

\paragraph{Bounding $\Vert I_{12}\Vert _\cH$.} The idea is to apply the Fuk-Nagaev inequality \cref{cor:confidence_bound} to the independent and identically distributed random variables 
\[ \zeta_i := \epsilon_i (\cov + \alpha \idop_\cH )^{-1/2}  \phi(X_i) .\]
We repeatedly apply \cref{assump:mom} to bound the expectation and moments of the $\zeta_i$. We get 
\begin{align*}
\mbe[\zeta_i] &= \mbe[\underbrace{\mbe[\epsilon_i | X_i]}_{=0}(\cov + \alpha \idop_\cH )^{-1/2}  \phi(X_i) ] = 0. 
\end{align*}
Moreover, 
\begin{align*}
\mbe[\Vert \zeta_i\Vert ^2_\cH] &= \mbe[\mbe[\epsilon^2_i | X_i] \; \Vert (\cov + \alpha \idop_\cH )^{-1/2}  \phi(X_i)\Vert ^2_\cH] \\
&< \sigma^2 \mbe[\tr((\cov + \alpha \idop_\cH )^{-1}  \phi(X_i) \otimes \phi(X_i)) ]\\
&= \sigma^2 \cN(\alpha) \\
&=: \tilde \sigma^2. 
\end{align*}
Since $q>2$, we find 
\begin{align*}
\mbe[\Vert \zeta_i\Vert ^q_\cH] &= \mbe[ \mbe[ |\epsilon_i|^q | X_i]  \; \Vert (\cov + \alpha \idop_\cH )^{-1/2}  \phi(X_i)\Vert ^{q}_\cH  ] \\
&< Q \; \mbe[\Vert (\cov + \alpha \idop_\cH )^{-1/2}  \phi(X_i)\Vert ^{2}_\cH \Vert (\cov + \alpha \idop_\cH )^{-1/2}  \phi(X_i)\Vert ^{q-2}_\cH ]  \\
&\leq Q\; \kappa^{q-2} \left( \frac{1}{\alpha}\right)^{\frac{q-2}{2}} \cN(\alpha)\\
&=: \tilde Q . 
\end{align*}
Since $\frac{1}{n}\sum_{i=1}^n\zeta_i = I_{12}$, applying \cref{cor:confidence_bound}  gives with confidence $1-\delta$
\begin{align}
\label{eq:I12}
\Vert I_{12}\Vert _\cH &\leq \max
    \left\{
    \left(\frac{2 c_1  \tilde Q}{\delta n^{q-1}} \right)^{1/q}  ,
    \tilde \sigma \; \sqrt{\frac{\log(2c_1/\delta)}{ c_2 n} }
    \right\} \nonumber \\
&=\max \left\{
    \left(\frac{2 c_1  \kappa^{q-2} Q}{\delta n^{q-1}} \right)^{1/q}  \left( \frac{1}{\alpha}\right)^{\frac{q-2}{2q}} \cN(\alpha)^{1/q} ,
     \sigma \; \sqrt{\frac{\cN(\alpha)\log(2c_1/\delta)}{ c_2 n} }
    \right\}   \nonumber  \\
&\leq  c_{12} \sqrt{\cN(\alpha)} \;
    \max \left\{ \left(\frac{1}{\delta n^{q-1}} \right)^{1/q}  \left( \frac{1}{\alpha \cN(\alpha)}\right)^{\frac{q-2}{2q}}  ,
      \sqrt{\frac{\log(2c_1/\delta)}{  n} }
    \right\} ,
\end{align}
with $ c_{12} = \max\{ (2 c_1  \kappa^{q-2} Q)^{1/q} , \sigma/ \sqrt{c_2}\} $.

\paragraph{Bounding $\Vert I_{13}\Vert _\cH$.} Using the definition of $f_\alpha$ from \eqref{eq:solution_rkhs}, 
we see that by \eqref{eq:src}, we have 
\begin{align*}
I_{13}&=    (\cov + \alpha \idop_\cH )^{-1/2} (\inc^* f_\star - \cov  \regcov I_\pi^* f_\star )\nonumber \\
&=  (\cov + \alpha \idop_\cH )^{-1/2} (\idop_\cH - \cov  \regcov ) I_\pi^* f_\star \nonumber \\
&= (\cov + \alpha \idop_\cH )^{-1/2} (\idop_\cH - \cov  \regcov ) I_\pi^* \cT_\pi^\nu f \nonumber \\
&= I_\pi^*  (\cT_\pi + \alpha \idop_{L^2} )^{-1/2}
( \idop_{L^2} -  \cT_\pi ( \cT_\pi + \alpha \idop_{L^2})^{-1})\cT_\pi^\nu f
\nonumber \\
&= \tilde{U} \cT_\pi^{1/2}  (\cT_\pi + \alpha \idop_{L^2} )^{-1/2}
( \idop_{L^2} -  \cT_\pi ( \cT_\pi + \alpha \idop_{L^2})^{-1})\cT_\pi^\nu f,
\end{align*}
for some $f \in L^2(\pi)$ satisfying $\Vert f\Vert _{L^2(\pi)} \leq R$
and the partial isometry $\tilde{U}: L^2(\pi) \to \mathcal{H}$
given by \eqref{eq:polar_decomposition}.
This gives 
\begin{align}
    \label{eq:I13}
    \Vert I_{13}\Vert _{\cH}  
    &\leq   
    R \sup_{ 0 <t \leq \kappa^2} |  \alpha t^{\nu+1/2}(t+\alpha)^{-3/2} | \nonumber \\
    &\leq R 
    \sup_{ 0 <t \leq \kappa^2} |  \alpha t^{\nu}(t+\alpha)^{-1} |
    \sup_{ 0 <t \leq \kappa^2} |   t^{1/2}(t+\alpha)^{-1/2} | \nonumber \\
    &\leq
    R 
    \sup_{ 0 <t \leq \kappa^2} |  \alpha t^{\nu}(t+\alpha)^{-1} | \nonumber \\
    &\leq 
    R \max\{1, {\kappa^{2(\nu - 1)}}\} 
    \, \alpha^{\min\{\nu, 1\} }
    = c_{13} \, \alpha^{\min\{\nu, 1\} }
    ,
\end{align}
where we use \cref{lem:qualification} for the last inequality and set
$c_{13} := R \max\{1, {\kappa^{2(\nu - 1)}}\} $.

\paragraph{Collecting all terms.} Collecting
the bounds 
\eqref{eq:I11}, \eqref{eq:I12}, \eqref{eq:I13}, together with \eqref{eq:3norms}
and taking a union bound gives with confidence $1-\delta$
\begin{align}
\label{eq:I1bound.final}
\Vert I_1\Vert _{L^2(\pi)} 
&\leq  4 \left( \Vert I_{11}\Vert _\cH  + \Vert I_{12}\Vert _\cH  + \Vert I_{13}\Vert _\cH \right)  \nonumber \\
&\leq c_{\kappa, \nu ,q, R} \; \left(  \alpha^{\min\{\nu, 1\} } +  \frac{ \log(4/\delta)}{\sqrt{\alpha} n}  + \sqrt{\frac{\cN(\alpha) \log(4/\delta)}{n}} +  \sqrt{\cN(\alpha)} \cdot    \eta(\delta, n, \alpha) \right), 
\end{align}
where we set 
\[ \eta(\delta, n, \alpha) := \max \left\{ \left(\frac{1}{\delta n^{q-1}} \right)^{1/q}  \left( \frac{1}{\alpha \cN(\alpha)}\right)^{\frac{q-2}{2q}}  ,
      \sqrt{\frac{\log(4c_1/\delta)}{  n} }
          \right\} , \]
and with $c_{\kappa, \nu, q, R}:=4\max\{ c_{11}, c_{12}, c_{13}\}$.


\subsubsection{Bounding the second term $I_2$} 

We write 
again $\cC_\alpha:= \cov + \alpha \idop_\cH  $, $ \widehat{\cC}_\alpha := \empcov + \alpha \idop_\cH$ and find with 
\cref{lem:cordes}
and
\cref{cor:concentration_inverse_product_under_evd} 
with confidence $1-\delta$
\begin{align*}
    \Vert I_2\Vert _{L^2(\pi)} &=\Vert \alpha \inc 
    ( \empcov + \alpha \idop_\cH )^{-1} f_\alpha\Vert _{L^2(\pi)} \\
    &= \alpha\Vert  \inc \cC_\alpha^{-1/2} \cC_\alpha^{1/2} \widehat{\cC}_\alpha^{-1/2} \widehat{\cC}_\alpha^{-1/2} \cC_\alpha^{1/2} \cC_\alpha^{-1/2}f_\alpha \Vert _{L^2(\pi)} \\
    &= \alpha 
    \Vert  \cC_\pi^{1/2} \cC_\alpha^{-1/2} \cC_\alpha^{1/2} \widehat{\cC}_\alpha^{-1/2} \widehat{\cC}_\alpha^{-1/2} \cC_\alpha^{1/2} \cC_\alpha^{-1/2}f_\alpha \Vert _{\mathcal{H}} \\
    &\leq 4\alpha \; \Vert  \cC_\alpha^{-1/2}f_\alpha \Vert _\cH, 
\end{align*}
where we again use \eqref{eq:polar_decomposition} and 
$\Vert   \cov ^{1/2} \cC_\alpha^{-1/2} \Vert _{\cL(\cH)} \leq 1$.
Hence, by \eqref{eq:src} and the definition of $f_\alpha$, we have with 
confidence $1-\delta$
\begin{align}
\label{eq:I2bound}
\Vert I_2\Vert _{L^2(\pi)} 
&\leq  4\alpha \; \Vert  \cC_\alpha^{-1/2}f_\alpha \Vert _\cH \nonumber \\ 
&\leq 4 R \sup_{0<t \leq \kappa^2} |\alpha (t + \alpha)^{-3/2} t^{\nu + 1/2}| \nonumber \\
 &\leq 4R \max\{1, {\kappa^{2(\nu - 1)}}\} \,\alpha^{\min\{\nu , 1\}}  
 = 4 c_{13} \,\alpha^{\min\{\nu , 1\}},
\end{align}
where we repeat the computation from \eqref{eq:I13}.


\subsubsection{Final variance bound}

Collecting now \eqref{eq:variance}, \eqref{eq:I1bound.final} and \eqref{eq:I2bound} and taking a union bound gives the final bound for the variance. With confidence $1-\delta$, we obtain 

\begin{align*}
\Vert I_\pi( \widehat f_\alpha -   f_\alpha)\Vert _{L^2(\pi)} &\leq 
\tilde c_{\kappa, \nu ,q, R} \; \left(  \alpha^{\min\{\nu, 1\} } +  \frac{ \log(8/\delta)}{\sqrt{\alpha} n}  + \sqrt{\frac{\cN(\alpha) \log(8/\delta)}{n}} +  \sqrt{\cN(\alpha)} \cdot    \eta(\delta, n, \alpha) \right), 
\end{align*}
where we set 
\[ \eta(\delta, n, \alpha) := \max \left\{ \left(\frac{1}{\delta n^{q-1}} \right)^{1/q}  \left( \frac{1}{\alpha \cN(\alpha)}\right)^{\frac{q-2}{2q}}  ,
      \sqrt{\frac{\log(8c_1/\delta)}{  n} }
          \right\} , \]
and with $\tilde c_{\kappa, \nu,q, R} := 8\max\{c_{11}, c_{12}, c_{13} \}$.


\subsubsection{Bounding the bias $ I_\pi f_\alpha - f_\star$}

The bias can be bounded in the same way as in Section \ref{sec:bounding-bias}: 
\begin{align*}
    \norm{ \inc f_\alpha   - f_\star}_{L^2(\pi)}
    \leq R \alpha^{\min\{ \nu, 1 \}}.
\end{align*}

Combining the variance bound and the bias bound yields the result.

\subsection{Proof of \cref{cor:capacity_rates_new}}

We first need a standard lemma that allows us to
bound the effective dimension based on the eigenvalue
decay of $\mathcal{T}_\pi$ due to
\cite[][Lemma 11]{fischer2020sobolev}.

\begin{lemma}[Eigenvalue decay]
\label{lem:evd_new}
    Assume \eqref{eq:evd_new}. 
    Then there exists constant $\tilde{D} > 0$ such that
    \begin{equation*}
        \mathcal{N}(\alpha) 
        \leq 
        \tilde{D} \alpha^{-p}
        \quad \alpha > 0.
    \end{equation*}
\end{lemma}

For simplicity, we will use the symbols
$\lesssim$ for inequalities
which hold up to a nonnegative multiplicative 
constant that does not depend on $n$ and $\delta$.

We first note that for every fixed $\delta \in (0,1)$,
a short calculation shows that 
there exists an $\widetilde{n}_0 \in \N$ such that with the choice
\begin{equation*}
\alpha(n, \delta) := 
\left( \frac{\log(8 c_1 / \delta)}{n} \right)^{\frac{1}{2\min\{\nu ,1\} +p}}
\end{equation*}
the condition \eqref{eq:ass:eff:dim_new} is satisfied
(since we have $n \alpha(n, \delta) \to \infty$ and 
$\sqrt{n} \alpha(n, \delta)^{(1+p)/2} \to \infty$
for $n \to \infty$).
We may hence apply \cref{prop:excess_risk_capacity_new}
and with confidence $1-\delta$, we have
\begin{align}
    \label{eq:excess_bound_proof}
\Vert I_\pi \widehat f_\alpha -   f_\star\Vert _{L^2(\pi)} &\leq 
c \; \left(  \alpha^{\min\{\nu, 1\} } +  \frac{ \log(8/\delta)}{\sqrt{\alpha} n}  + \sqrt{\frac{\cN(\alpha) \log(8/\delta)}{n}}
+  \sqrt{\cN(\alpha)} \cdot    \eta(\delta, n, \alpha) \right), 
\end{align}
with
\begin{equation*}
    \eta(\delta, n, \alpha) 
    = \max \left\{ \left(\frac{1}{\delta n^{q-1}} \right)^{1/q}  \cdot
    \left( \frac{1}{\alpha \cN(\alpha)}\right)^{\frac{q-2}{2q}} ,
    \sqrt{\frac{\log(8c_1/\delta)}{n} }
    \right\}
\end{equation*} for all $n \geq \tilde{n}_0$ with a
constant $c >0$ not depending
on $\delta$ and $n$.

We now show
there exists $\overline{n}_0 \in \N$ such that
\begin{equation}
    \label{eq:subgaussian_dominant}
    \sqrt{\cN(\alpha(n, \delta))} \cdot
    \eta(\delta, n, \alpha(n, \delta))  
    \lesssim 
    \sqrt{\frac{\log(8c_1/\delta)}{n \alpha(n, \delta)^p}}
\end{equation}
for all $n \geq \overline{n}_0$, i.e.,
the subgaussian term asymptotically dominates the regularized
Fuk--Nagaev term under the choice $\alpha(n, \delta)$.
In fact, note that under $\mathcal{N}(\alpha) \leq \tilde{D} \alpha^{p}$
ensured by \cref{lem:evd_new}, we have
$$
 \sqrt{\cN(\alpha)}\left(\frac{1}{\alpha \cN(\alpha)}\right)^{\frac{q-2}{2q}} = \left(\frac{1}{\alpha}\right)^{\frac{q-2}{2q}}\cN(\alpha)^{p/q} \leq \left(\frac{1}{\alpha}\right)^{\frac{q-2}{2q}}\left(\frac{\tilde{D}}{\alpha}\right)^{p/q},
$$
and we hence have
\begin{equation}
\label{eq:N_eta_simplified}
\sqrt{\cN(\alpha)} \cdot
    \eta(\delta, n, \alpha) 
    \lesssim 
    \max \left\{ \left(\frac{1}{\delta n^{q-1}} \right)^{1/q}  \cdot
    \left(\frac{1}{\alpha}\right)^{\frac{q-2}{2q}}\left(\frac{\tilde{D}}{\alpha}\right)^{p/q} ,
    \sqrt{\frac{\log(8c_1/\delta)}{n\alpha^p} }
    \right\}.
\end{equation}
We insert the definition of $\alpha = \alpha(n, \delta)$
into \eqref{eq:N_eta_simplified}
and isolate the exponents corresponding to $1/n$
in both expressions inside of the above maximum.
We obtain the exponent
\begin{equation}
    \label{eq:exponent1}
    \frac{q-1}{q} - \frac{1}{2\min\{\nu ,1\} + p} \cdot 
    \left( 
    \frac{q-2}{2q} + \frac{p}{q}
    \right)
    =
    \frac{2(q-1)(2\min\{\nu ,1\} + p)
     - (q-2+2p)}{2q(2\min\{\nu ,1\} + p)}
\end{equation}
for the first expression and
\begin{equation}
    \label{eq:exponent2}
    \frac{1}{2} - \frac{p/2}{2\min\{\nu ,1\} + p}
    = 
    \frac{\min\{\nu ,1\}}{2\min\{\nu ,1\} + p}
    =
    \frac{2q\min\{\nu ,1\} }{2q(2\min\{\nu ,1\} + p)}
\end{equation}
for the second expression. 
We show that the difference between
the numerator of $\eqref{eq:exponent1}$ and
the numerator of $\eqref{eq:exponent2}$ is nonnegative.
We have
\begin{align*}
    2(q-1)(2\min\{\nu ,1\} + p)
     - q+2-2p - 2 q \min\{\nu ,1\} 
     &=
     (4q-4)\underbrace{\min\{\nu ,1\}}_{\geq 1/2} + 
     \underbrace{2p (q-2)}_{\geq 0} + 2 - q \\
     &\geq q - 2 + 2 - q = 0,
\end{align*} 
where we
use $\nu \geq 2$, $q \geq 3$ and $p \in (0,1)$,
hence proving \eqref{eq:subgaussian_dominant}.

Consequently,
for $n \geq \max\{ \widetilde{n}_0, \overline{n}_0 \}$ and with
\cref{lem:evd_new} and
\eqref{eq:subgaussian_dominant},
the bound provided by \eqref{eq:excess_bound_proof} 
now reduces to 
\begin{align*}
    \Vert I_\pi \widehat f_{\alpha(n, \delta)} -   f_\star\Vert _{L^2(\pi)}
    &\lesssim
        \alpha(n, \delta)^{\min\{\nu, 1\} } 
        +
        \frac{ \log(8c_1/\delta)}{\sqrt{\alpha(n, \delta)} n}  
        +
        \sqrt{ \frac{\log(8c_1/\delta)}{n \alpha(n, \delta)^p} } \\
    & 
    \lesssim
    \left( \frac{\log(8 c_1 / \delta)}{n} \right)^{\frac{\min\{\nu ,1\}}{2\min\{\nu ,1\} +p}}
    +
    \left( \frac{\log(8 c_1 / \delta)}{n} \right)^{\frac{2\min\{\nu ,1\} + p - 1/2}%
    {2\min\{\nu ,1\} +p}} \\
    &
    \quad +
    \left( \frac{\log(8 c_1 / \delta)}{n} \right)^{\frac{\min\{\nu ,1\}}{2\min\{\nu ,1\} + p}}
    \\
    &\lesssim
    \left( \frac{\log(8 c_1 / \delta)}{n} \right)^{\frac{\min\{\nu ,1\}}{2\min\{\nu ,1\} +p}}
\end{align*}
with confidence $1-\delta$,
where we use $\sqrt{\log(8 / \delta )} \leq \log(8/\delta) 
\leq \log(8 c_1/ \delta)$ since
$c_1 \geq 1$,
and bound the exponent of the middle term in the second step using $\nu \geq 1/2$.

\section{Excess risk bound for polynomial confidence regime}
\label{app:polynomial_regime}

We now prove a simplified capacity-free excess risk bound based
on \cref{prop:excess_risk_bound} for the polynomial
confidence regime $\delta \in D_2(n, q)$.

\begin{corollary}[Polynomial confidence regime]
\label{cor:high_confidence}
Let \eqref{eq:mom} and \eqref{eq:src} be satisfied. There exists a constant $c>0$ and a subset $\widetilde{D}_2(n,q) \subseteq D_2(n,q)$, defined in the proof,  such that 
for all $\delta \in \widetilde{D}_2(n, q)$, we have
\begin{align}
\label{eq:second-cor}
\norm{\inc  \widehat f_{\alpha_2(n, \delta)} - f_\star}_{L^2(\pi)}
    &\leq c \cdot \alpha_2(n, \delta)^{\min\{ \nu, 1 \}}  
\end{align}
with confidence $1-\delta$, where 
\[ 
  \alpha_2(n, \delta) := 
  \max \left\{   \left(\frac{1}{\delta n^{q-1}} \right)^{\frac{2}{q(2\min\{\nu ,1\}+1)}}, \frac{\log(6/\delta)}{\sqrt{n}}, \kappa^2 \right\}.
\]
\end{corollary}

\begin{proof}
Let $\delta \in D_2(n,q)$. We have from \cref{prop:excess_risk_bound} with confidence $1-\delta$
\begin{align*}
\norm{\inc  \widehat f_\alpha - f_\star}_{L^2(\pi)}
    &\leq R \alpha^{\min\{ \nu, 1 \}}  +  \frac{C_\diamond}{\sqrt \alpha} 
    \left(    \frac{\log(6/\delta)}{n} + 
  \sqrt{ \frac{\alpha^{2 \min\{\nu , 1\}} \log(6/\delta)}{n} }  
    +   \left(\frac{Q}{\delta n^{q-1}} \right)^{1/q} \right) .
\end{align*}
Note that $\delta \in D_2(n,q)$ and $c_1 \geq 1$ ensure that 
\[ \sqrt{\frac{\log(6/\delta)}{n} }
\leq \sqrt{\frac{\log(6c_1/\delta)}{n} }
                            \leq \frac{1}{\sigma} \;  \left(\frac{Q}{\delta n^{q-1}} \right)^{1/q}.
\]
Hence, since $\alpha_2(n, \delta) \leq \kappa^2$, we find 
\[ \sqrt{ \frac{\alpha^{2 \min\{\nu , 1\}} \log(6/\delta)}{n} }  
\leq \frac{\kappa^{2\min\{\nu , 1\} }}{\sigma}\;  \left(\frac{Q}{\delta n^{q-1}} \right)^{1/q}. \]
This leads to the bound  
\begin{align*}
\norm{\inc  \widehat f_\alpha - f_\star}_{L^2(\pi)}
    &\leq R \alpha^{\min\{ \nu, 1 \}}  +  \frac{C_\diamond}{\sqrt \alpha} 
    \left(    \frac{\log(6/\delta)}{n} + 
  c_7 \;   \left(\frac{1}{\delta n^{q-1}} \right)^{1/q} \right) ,
\end{align*}
holding with confidence $1-\delta$ and where we set $c_7:=Q^{1/q}\;\left(1+\frac{\kappa^{2\min\{\nu , 1\} }}{\sigma} \right)  $.

Next, we observe that condition \eqref{eq:alpha-coice} implies 
\[  \frac{\log(6/\delta)}{n} \leq \frac{1}{C_\kappa} \frac{\alpha}{\sqrt n} .  \]
Further assuming
\begin{equation}
\label{eq:alpha:reg2-1}    
\alpha \geq \left(\frac{1}{n}\right)^{ \frac{1}{2\min\{\nu, 1\} -1}},
\end{equation}
then 
\[ \frac{C_\diamond}{C_\kappa \sqrt \alpha} \cdot \frac{\alpha}{\sqrt n} 
\leq  \frac{C_\diamond}{C_\kappa } \;  \alpha^{\min\{\nu, 1\}} \;. \]
The upper bound for the excess risk reduces then to 
\begin{align*}
\norm{\inc  \widehat f_\alpha - f_\star}_{L^2(\pi)}
    &\leq c_8\;   \alpha^{\min\{ \nu, 1 \}}  +       \frac{c_9}{\sqrt \alpha} \;  
    \left(\frac{1}{\delta n^{q-1}} \right)^{1/q} .
\end{align*}
where we introduce $c_8=R +  \frac{C_\diamond}{C_\kappa } $ and $c_9=c_7 C_\diamond$. 

To ensure that the remaining variance part is of the same order as the approximation error part, we need to choose 
\begin{equation}
\label{eq:alpha:reg2-2}
\alpha \geq c_{10} \;  \left(\frac{1}{\delta n^{q-1}} \right)^{\frac{2}{q(2\min\{\nu ,1\}+1)}}, 
\end{equation}
with $c_{10} = \left( \frac{c_9}{c_8}\right)^{\frac{2}{2\min\{\nu ,1\}+1}}$. This finally gives with confidence $1-\delta$ 
\begin{align*}
\norm{\inc  \widehat f_\alpha - f_\star}_{L^2(\pi)}
    &\leq 2c_8\;   \alpha^{\min\{ \nu, 1 \}}   ,
\end{align*}
provided conditions $\delta \in D_2(n, q)$ \eqref{eq:alpha-coice}, \eqref{eq:alpha:reg2-1}, \eqref{eq:alpha:reg2-2} and $\alpha \leq \kappa^2$ hold. 
To simplify the conditions for $\alpha$, we observe that for all $q\geq 3$, $\nu \leq 1$ and $\delta \in (0,1)$ we have 
\[   \left( \frac{1}{n} \right)^{^{ \frac{1}{2\min\{\nu, 1\} -1}}} 
\leq \left( \frac{1}{\delta n^{q-1}} \right)^{\frac{2}{q(2\min\{\nu ,1\}+1)}} . \]
As a result, condition \eqref{eq:alpha:reg2-2} implies \eqref{eq:alpha:reg2-1}  with an appropriate constant. To sum up, the regularization parameter needs to satisfy 
\begin{equation*}
\label{eq:alpha2}
  \alpha \geq  
 c_{11} \cdot \max \left\{   \left(\frac{1}{\delta n^{q-1}} \right)^{\frac{2}{q(2\min\{\nu ,1\}+1)}}, \frac{\log(6/\delta)}{\sqrt{n}}\right\} ,
\end{equation*}
with $c_{11}= \max\{1,c_{10} , C_\kappa\} $. 

To ensure that $\alpha$ remains bounded by $\kappa^2$ it is sufficient to choose $n$ sufficiently large: 
\[  n \geq n_0(\delta):= c_{11}\cdot \max\left\{ \kappa^{-2}\log^2(6/\delta)  , 
 \delta^{-\frac{1}{q-1}} \kappa^{-\gamma}\right\}, \]
 with $\gamma = \frac{2q(2\min\{\nu, 1\}+1)}{q-1} $. Recall that $\delta \in D_2(n,q)$ requires 
 \[   n \leq n_{max}(\delta):= \left( \frac{Q^2}{\sigma^{2q}}\right)^{\frac{1}{q-2}} 
\delta^{-\frac{2}{q-2}}\cdot \log(6c_1/\delta)^{-\frac{q}{q-2}} ,\]
so we need to restrict $D_2(n, q)$ such that both conditions for $n$ are met. Letting now 
$\delta \in \widetilde{D}_2(n,q)$ with
\[\widetilde{D}_2(n,q) := \{ \delta \in D_2(n,q) \; :\; n_0(\delta) \leq n_{max}(\delta)\} , \]
then \eqref{eq:second-cor} holds with confidence $1-\delta$. 

\end{proof}

\section{Concentration bounds}
\label{app:concentration_bounds}

We collect the concentration bounds used in the main text of this work
and prove \cref{prop:fn_modified}.

\subsection{Hoeffding inequality in Hilbert spaces}

The following bound is classical and follows as a special
case of \cite[Theorem 3.5]{Pinelis94}, see also 
\citep[Section A.5.1]{Mollenhauer2022Phd}.

\begin{proposition}[Hoeffding inequality]
\label{prop:hoeffding}
Let $\xi, \xi_1, \dots \xi_n$ be independent random variables 
taking values in a Hilbert space $\X$ such that
$\E[\xi] = 0$ and $\norm{\xi}_\X \leq L$ 
almost surely.
Then for all $t>0$, we have
\begin{equation*}
    \P\left[ 
        \Norm{\frac{1}{n} \sum_{i=1}^n\xi_i}_X  \geq t
    \right] 
    \leq
    2\exp \left(
    - \frac{nt^2}{2L^2}
    \right).
\end{equation*}
\end{proposition}

\subsection{Bennett inequality in Hilbert spaces}

We now give a version of a Bennett-type inequality
going back to \citep[Theorem 3.4]{Pinelis94}.
We use the confidence bound version as derived by
\citep[Lemma 2]{Smale2007}.

\begin{proposition}[Bennett inequality]
    \label{prop:bennett}
    Let $\xi, \xi_1, \dots \xi_n$ be independent and identically distributed 
    random variables 
    taking values in a Hilbert space $\X$ such that
    $\norm{\xi}_\X \leq L$ almost surely and
    $\sigma^2:= \E[\norm{\xi}^2_\X]$.
    Then for any $\delta \in (0,1)$, we have
    \begin{equation*}
    \Norm{ \frac{1}{n} \sum_{i=1}^n \xi_i - \E[\xi] }_\X
    \leq
    \frac{2 L \log(2/\delta)}{n}
    +
    \sqrt{\frac{2 \sigma^2 \log(2/\delta)}{n}}
    \end{equation*}
    with probability at least $1-\delta$.
\end{proposition}

\subsection{Fuk--Nagaev inequality in Hilbert spaces}
\label{app:fuk-nagaev}

The bound in \cref{prop:fn_modified} given in the main text
is a sharper version of a more general
bound given by \citep{Yurinsky1995} for
general normed spaces, which we state here for completeness.

\yurinsky{
The original proof of the result below
relies on an inconsistent exponential moment bound, 
which we address by providing 
an alternative bound alongside a detailed proof.
We were made aware of this fact by Christian Fiedler \cite{christian},
who contributed to the alternative proof presented here.
}

\begin{proposition}[Fuk--Nagaev inequality, \citep{Yurinsky1995}, Theorem 3.5.1]
\label{prop:fn_yurinsky}
Let $\xi_1, \dots \xi_n$ be independent and identically distributed random variables taking values
in a normed space $\X$ with measurable norm such that
\begin{equation}
    \label{eq:fn_condition}
    \E[\xi_i] = 0, \quad
    \sum_{i=1}^n \E\left[ \norm{\xi_i}_\X^2 \right] < B^2
    \quad \text{and} \quad
    \sum_{i=1}^n \E\left[\norm{\xi_i}_\X^q\right] < A
\end{equation}
for some constants $B^2 >0$, $A >0$ and $q \in \N$, $q \geq 3$.
Then there exist two universal constants $c_1 >0$
and $c_2> 0$ depending only on $q$, such that
with $S_n := \sum_{i=1}^n \xi_i$ we have for every $t>0$ that
\begin{equation*}
    \P \big[ \,
        \norm{S_n}_\X - \E\left[ \norm{S_n}_\X \right] \geq tB
        \big]
        \leq 
        c_1 \left(  
        \frac{A}{B^{q}t^q}{}
        + 
        \exp(-c_2 t^2)
        \right).
\end{equation*}
\end{proposition}

Compared to the result above, 
\cref{prop:fn_modified} removes the excess term 
$\E[\norm{S_n}_\X]$ on the left hand side
by making use of the geometry of the Hilbert space norm
(see \cite{Pinelis1986Remarks, Pinelis94}). 

We now prove \cref{prop:fn_modified}
and emphasize that it
can be directly extended to 
random variables in \emph{$(2,D)$-smooth Banach spaces} 
by incorporating arguments from \citep{Pinelis94} with adjusted constants.
\yurinsky{
We provide the proof in three parts:
In \cref{sec:exponential_moment_bound}, we give an 
exponential moment bound.
\cref{sec:fn_core_proof} contains the
core proof of \cref{prop:fn_modified} and incorporates
the exponential moment bound
into a Chernoff bound that is obtained via 
a truncation argument. 
\cref{sec:chernoff_optimization} provides the final
optimization of the Chernoff bound.
}

\yurinsky{
\subsubsection{Exponential moment bound}
\label{sec:exponential_moment_bound}

We now provide a sharpened version of \cite[][Lemma 3.5.1]{Yurinsky1995}
that allows to bound the exponential moments of sums of random 
variables in Hilbert spaces.
The proof
of \cite[][Lemma 3.5.1]{Yurinsky1995} is lacking details
and is inconsistent,
as it generally seems to require bounded scaling factors $h >0$
in the exponent, but the
result is stated for all $h>0$.
Our proof holds for all $h>0$ and gives a
slightly adjusted bound.

\begin{lemma}
\label{lem:exponential_moments}
Let $\xi_1, \dots \xi_n$ be independent random variables taking values
in a separable Hilbert space $\X$ such that the conditions
\eqref{eq:fn_condition} are satisfied and additionally
$\norm{\xi_i}_\X \leq L$ holds almost surely for some $L >0$ and
all $1\leq i \leq n$. Then for
all $h > 0$, 
we have
\begin{equation*}
    \E[\cosh( h \norm{S_n}_\X)] 
    \leq
    \exp\left(
            \tilde{c}_2 h^2 B^2 
            + \tilde{c}_q A h^q e^{hL}
            \right),
\end{equation*}
where the positive multiplicative constants $\tilde{c}_2$ 
and $\tilde{c}_q$ only depend on $q$.
\end{lemma}

The upper bound of 
\cref{lem:exponential_moments} differs from
\cite[][Lemma 3.5.1]{Yurinsky1995} 
in the sense that 
it removes an additive constant in the exponent.
Furthermore, our constant $\tilde{c}_2$ of the quadratic 
term depends on $q$,
while $\tilde{c}_2 = 1/2$ for all $q$
in \cite[][Lemma 3.5.1]{Yurinsky1995},
which seems to conflict with unbounded $h >0$.

\begin{proof}
We have
\begin{align*}
    \E[\cosh( h \norm{S_n}_\X)] 
    &\leq 
    \prod_{i=1}^n \E \left[ 
        e^{h \norm{\xi_i}_\X}  - h \norm{\xi_i}_\X
        \right] \\
    &\leq
    \exp\left(
        \sum_{i=1}^n 
        \E[ e^{h \norm{\xi_i}_\X} - 1 - h \norm{\xi_i}_\X ]
    \right)
\end{align*}
where the first inequality is due to 
\cite[Theorem 3]{Pinelis1986Remarks},
the second inequality follows from the fact that we have
$x \leq \exp(x-1)$ for all $x \in \R$.
We hence need to show
\begin{equation*}
    \sum_{i=1}^n 
    \E[ e^{h \norm{\xi_i}_\X} - 1 - h \norm{\xi_i}_\X ]
    \leq
    \tilde{c}_1 h^2 B^2 
    + \tilde{c}_q A h^q e^{hL}.
\end{equation*}

We start by expanding the term on the left hand side as
\begin{equation*}
    \sum_{i=1}^n
    e^{h \norm{\xi_i}_\X} - 1 - h \norm{\xi_i}_\X
    =
    \sum_{i=1}^n
    \sum_{j = 2}^\infty 
    \frac{h^j \norm{\xi_i}^j_\X }{j!}.
\end{equation*}
The term for the index $j=2$ is immediately bounded as
\begin{equation*}
    \E \left[
    \sum_{i=1}^n
    \frac{h^2 \norm{\xi_i}^2_\X }{2!}
    \right]
    \leq \frac{1}{2}h^2B^2.
\end{equation*}
Hence, the claim now reduces to showing
\begin{equation*}
    \E \left[
    \sum_{i=1}^n
    \sum_{j = 3}^\infty 
    \frac{h^j \norm{\xi_i}^j_\X }{j!}
    \right]
    \lesssim
    h^2 B^2 
    +
    A h^q e^{hL}
\end{equation*}
with constants only depending on $q$.
We split the inner sum over $j$ at the index $q$ and obtain
\begin{align*}
    \E \left[
    \sum_{i=1}^n
    \sum_{j = 3}^\infty 
    \frac{h^j \norm{\xi_i}^j_\X }{j!}
    \right]
    =
    \E \left[
    \sum_{i=1}^n
    \sum_{j = 3}^{q-1}
    \frac{h^j \norm{\xi_i}^j_\X }{j!}
    \right]
    +
    \E \left[
    \sum_{i=1}^n
    \sum_{j = q}^\infty
    \frac{h^j \norm{\xi_i}^j_\X }{j!}
    \right].
\end{align*}

We now proceed to bound both terms
on the right hand side individually.

\paragraph{Bounding the first summand.}
We address the first term as
\begin{align*}
    \E \left[
    \sum_{i=1}^n
    \sum_{j = 3}^{q-1}
    \frac{h^j \norm{\xi_i}^j_\X }{j!}
    \right]
    &\leq
    \E \left[
    \sum_{i=1}^n
    \sum_{j = 3}^{q-1}
    \frac{
    \frac{j-2}{q-2} \cdot h^q\norm{\xi_i}^q_\X 
    + \frac{q-j}{q-2} \cdot h^2\norm{\xi_i}^2_\X 
    }{j!}
    \right] \\
    &=
    \sum_{j = 3}^{q-1}
    \frac{1}{j!}
        \sum_{i=1}^n 
        \frac{j-2}{q-2} \cdot h^q \E[\norm{\xi_i}_\X^q]
        +
        \frac{q-j}{q-2} \cdot h^2\E[\norm{\xi_i}_\X^2] \\
    &= 
    \sum_{j = 3}^{q-1}
    \frac{1}{j!}
    \left(
    \frac{j-2}{q-2} \cdot h^q A
    +
    \frac{q-j}{q-2} \cdot h^2 B^2
    \right) \\
    &\lesssim
    h^q A + h^2 B^2
    \leq
    h^q A e^{hL} + h^2 B^2,
\end{align*}
where we apply \cref{lem:fromRio17Prop3.5}
with $x = h \norm{\xi_j}_\cX$, $a = 1$ and $\rho = j$ 
to each individual term in the sum and
use that $h>0$ and $L>0$ ensure $e^{hL}\geq 1$.
Note that the unspecified constants only depend on $q$.

\paragraph{Bounding the second summand.}
It remains to bound the second sum. Almost surely, we have
\begin{align*}
        \E \left[
    \sum_{i=1}^n
    \sum_{j = q}^\infty
    \frac{h^j \norm{\xi_i}^j_\X }{j!}
    \right]
    &=
    \E \left[
    \sum_{i=1}^n
    \frac{h^q\norm{\xi_i}^q_\X }{q!}
    +
    \frac{h^{q+1}\norm{\xi_i}^{q+1}_\X }{(q+1)!}
    \cdots
    \right] \\
    &=
    \E \left[
    \sum_{i=1}^n
    \frac{h^q\norm{\xi_i}^q_\X }{q!}
    \left( 
    1
    +
    \frac{h\norm{\xi_i}_\X }{(q+1)}
    +
    \frac{h^{2}\norm{\xi_i}^{2}_\X }{(q+1)(q+2)}
    \cdots
    \right)
    \right] \\
    &\leq
    \E \left[
    \sum_{i=1}^n
    \frac{h^q\norm{\xi_i}^q_\X }{q!}
    \left( 
    1
    +
    \frac{h^{1}L }{(q+1)}
    +
    \frac{h^{2}L^2 }{(q+1)(q+2)}
    \cdots
    \right)
    \right] \\
    &\leq 
    \frac{h^q}{q!} e^{hL}  \sum_{i=1}^n \E\left[\norm{\xi_i}_\X^q\right] 
    \leq 
    \frac{h^q}{q!} e^{hL}  A,
\end{align*}
proving the claim.
\end{proof}

}

\subsubsection{Core proof of \cref{prop:fn_modified}}
\label{sec:fn_core_proof}

We can now prove \cref{prop:fn_modified} by
modifying the truncation argument and Chernoff bound given by
\citep[Section 3.5.2]{Yurinsky1995} in combination with
\cref{lem:exponential_moments}.
Let us assume that \eqref{eq:fn_condition} holds.
For $L >0$, 
we introduce the truncated random variables
\begin{equation*}
\tilde{\xi}_i := \xi_i \, \mathds{1}_{ \, [\|\xi_i\|_{\X} \leq L \,]}, \qquad 
\tilde{S}_n := \sum_{i=1}^n \tilde{\xi}_i.
\end{equation*}
{
In contrast to $S_n$, the truncated sum $\tilde S_n$ is not necessarily centered. 
Here, our version of the proof differs from
the proof of \cite[Theorem 3.5.1]{Yurinsky1995}.
The original
proof approximates the centered norm 
$ \norm{S_n}_{\mathcal{X}} - \E[ \norm{S_n}_\mathcal{X} ]$
with the truncated centered norm 
$ \norm{\tilde S_n}_{\mathcal{X}} - \E[ \norm{\tilde S_n}_\mathcal{X} ]$, leading to the excess term 
$\E[ \norm{S_n}_\mathcal{X}]$ in the event for which
the tail is bounded.
In contrast, we perform the centering directly in the norm:
we approximate $\norm{S_n}_{\mathcal{X}}$
with $\norm{\tilde S_n - \E[ \tilde S_n ] }_\mathcal{X}$.

We first note that the difference between $\norm{\tilde S_n}_{\mathcal{X}}$ and 
its centered version
$\norm{ \tilde S_n - \E[ \tilde S_n]}_{\mathcal{X}}$
can be bounded conveniently:
\begin{align}
    \label{eq:truncation_centering}
    \Absval{ \norm{\tilde S_n}_{\mathcal{X}}  
    - \norm{\tilde S_n - \E[ \tilde S_n]}_{\mathcal{X} } }
    &\leq 
    \norm{ \E[ \tilde S_n] }_{\mathcal{H}}
    =
    \norm{ \E[\tilde S_n] - \E[S_n] }_{\mathcal{X}}
    \nonumber \\
    & \leq 
    \sum_{i=1}^n \E \left[\norm{ \tilde \xi_i -\xi_i }_{\mathcal{X}}\right] 
    \nonumber \\
    &= 
    \sum_{i=1}^n \E \left[      
        \mathds{1}_{ \, [\|\xi_i\|_{\X} > L \,]}  \, 
         \norm{ \xi_i }_{\mathcal{X}} \, 
        \ \right]
    \nonumber \\
    &= 
    \sum_{i=1}^n \E \left[      
        \mathds{1}_{ \, [\|\xi_i\|_{\X} > L \,]} 
        \frac{\norm{ \xi_i }^{q-1}_{\mathcal{X}}}{\norm{ \xi_i }^{q-1}_{\mathcal{X}}} \, 
         \norm{ \xi_i }_{\mathcal{X}} \, 
        \ \right]
    \nonumber \\
    & \leq \frac{A}{ L^{q-1}}.
\end{align}

We now verify the conditions \eqref{eq:fn_condition}
for the centered sum $\tilde S_n - \E[ \tilde S_n]$.
We obviously have
\begin{equation*}
    \sum_{i=1}^n 
    \E [\norm{ \tilde \xi_i - \E[ \tilde \xi_i ]  }^2_{\mathcal{X}} ] 
    \leq 
    \sum_{i=1}^n \E [ \norm{ \tilde \xi_i  }^2_{\mathcal{X}} ] 
    \leq
    \sum_{i=1}^n \E [ \norm{ \xi_i  }^2_{\mathcal{X}} ]
    \leq B^2.
\end{equation*}
Furthermore,
from Minkowski's inequality followed by
$(a+b)^q \leq 2^{q-1} (a^q+b^q)$ for $a,b \geq 0$ and
Jensen's inequality,
we analogously obtain
\begin{align*}
    \sum_{i=1}^n 
    \E [ \norm{ \tilde \xi_i - \E[ \tilde \xi_i ]  }^q_{\mathcal{X}} ] 
    &\leq
    \sum_{i=1}^n 
    \left(
    \E [ \norm{ \tilde \xi_i }^q_\mathcal{X} ]^{1/q}
    + 
    \E[ \norm{ \tilde \xi_i }_\mathcal{X} ] 
    \right)^q
    \\
    &\leq
    2^{q-1}
    \sum_{i=1}^n 
    \left(
    \E [ \norm{ \tilde \xi_i }^q_\mathcal{X} ]
    + 
    \E[ \norm{ \tilde \xi_i }_\mathcal{X} ]^q
    \right) \\
    &\leq
    2^{q}
    \sum_{i=1}^n \E [ \norm{ \tilde \xi_i  }^q_{\mathcal{X}} ]
    \leq
    2^{q} \sum_{i=1}^n \E [ \norm{ \xi_i  }^q_{\mathcal{X}} ]
    \leq 2^{q} A =: \tilde{A}.
\end{align*}

We now provide the final tail bound
for the norm of $S_n$ by approximating it
with the centered truncated sum $\tilde{S_n} - \E[\tilde S_n]$.
For every $t > 0$ and $h>0$ we have
\begin{align*}
    \P[\norm{S_n}_\X \geq tB ] 
    &\leq
    \P[ S_n \neq \tilde{S}_n ] +
    \P[\norm{\tilde{S}_n}_\X \geq tB ] \\
    &\leq 
    \frac{A}{L^q}
    + 
    \P[\norm{\tilde{S}_n}_\X \geq tB ] 
    && \text{(Markov's inequality)}\\
    &\leq 
    \frac{A}{L^q}
    +
    \P[\norm{\tilde{S}_n - \E[\tilde S_n] ]}_\X \geq tB - A / L^{q-1} ] 
    && \text{(by \eqref{eq:truncation_centering})} \\
    &\leq
    \frac{A}{L^q}
    +
    \exp\left(-htB + \frac{hA}{L^{q-1}}\right) \, \E[ \exp(h \norm{\tilde{S}_n - \E[\tilde{S}_n]}_\X) ]
    && \text{(Chernoff bound)} \\
    &\leq 
    \frac{A}{L^q}
    +
    \exp\left(-htB + \frac{hA}{L^{q-1}}\right) \,
    2 \E[ \cosh(h \norm{\tilde{S}_n - \E[\tilde{S}_n]}_\X) ]
    && \text{($\cosh(x) = (e^x + e^{-x}) / 2$)} \\
    &\leq
    \frac{A}{L^q}
    + 
    2 \exp\left( -htB + \frac{h A}{L^{q-1}} + 
            \tilde{c}_2 h^2 B^2 
            + \tilde{c}_q \tilde A h^q e^{hL} \right)
    && (\text{\cref{lem:exponential_moments}}) \\
    &\leq
    2\left( \frac{\tilde A}{L^q}
    + 
    \exp\left( -htB + \frac{h \tilde A}{L^{q-1}} 
            + \tilde{c}_2 h^2 B^2 
            + \tilde{c}_q \tilde A h^q e^{hL} \right) \right).
\end{align*}
We derive an upper bound
of the last expression in the parentheses 
over all choices of $L>0$ and $h>0$ \yurinsky{as shown in detail in
\cref{sec:chernoff_optimization}, giving}
\begin{equation}
\label{eq:fn_modified_proof}
        \P \big[ \,
        \norm{S_n}_\X \geq tB
        \big]
        \leq 
        c_1 \left(  
        \frac{\tilde A}{B^{q}t^q}
        + 
        \exp(-c_2 t^2)
        \right)
\end{equation}
for some positive constants $c_1$ and $c_2$ only depending on 
$q$.
Note that we absorb the factor $2^{q}$ from
$\tilde A = 2^{q}A$ into the generic constant $c_1$ in the resulting bound.
}


Whenever the $\xi_i$ satisfy $\E[\norm{\xi_i}_\X^2] = \sigma^2$
and $\E[\norm{\xi_i}_\X^q] = Q$ for some constants
$\sigma^2, Q >0$, we may substitute
$B = n \sigma^2$ and $A = n Q$ in \eqref{eq:fn_modified_proof}. 
Rearranging proves the bound given in \cref{prop:fn_modified}.

\yurinsky{
\subsubsection{Chernoff bound optimization}
\label{sec:chernoff_optimization}

The final bound we want to obtain is
of the form
\begin{equation}
    \label{eq:final_bound}
    \min_{h, L \geq 0}
    \frac{\tilde A}{L^q}
    + 
    \exp\left( -htB + \frac{h \tilde A}{L^{q-1}} 
            + \tilde{c}_2 h^2 B^2 
            + \tilde{c}_q \tilde A h^q e^{hL} \right)
    \leq
        c_1 \left(  
        \frac{\tilde A}{B^{q}t^q}{}
        + 
        \exp(-c_2 t^2)
        \right)
\end{equation}
with positive constants 
$c_1$ and $c_2$ only depending on $q$.
We now adapt the idea of the original optimization 
argument by \cite[][pp. 105]{Yurinsky1995}
to the setting of our \cref{lem:exponential_moments}.
We first introduce the shorthand
notation
\begin{equation*}
\Delta(t) := \tilde A / (B^q t^q)
\quad
\text{and}
\quad
\Lambda(t) := L/(tB).
\end{equation*}
With this notation, we rewrite the term on the
left hand side of \eqref{eq:final_bound} as
\begin{align}
    \frac{\Delta(t)}{\Lambda(t)^q}
    +
    \exp
    \left(
        - htB
        +
        htB \Delta(t) / \Lambda(t)^{q-1}
        +
        \tilde{c}_2 h^2 B^2
        +
        \tilde{c}_q
        (htB)^q
        \Delta(t)
        e^{ \Lambda(t) th B } \nonumber
    \right)
    \\ 
    =
     \frac{\Delta(t)}{\Lambda(t)^q}
    +
    \exp
    \left(
        htB
        \left( \Delta(t) / \Lambda(t)^{q-1} - 1 \right)
        +
        \tilde{c}_2 h^2 B^2
        +
        \tilde{c}_q
        (htB)^q
        \Delta(t)
        e^{ \Lambda(t) th B }
    \right).
    \label{eq:yurinsky_exponential_moment_updated2}
\end{align}

Note that for all $t >0$, 
$\Lambda(t)$ can be chosen independently of $\Delta(t)$
by adjusting the truncation level $L$ in our
optimization.
Note also that $\Delta(t)$ is exactly the first summand in
the right hand side of
the final desired bound \eqref{eq:final_bound}.
We will check the validity of the bound of the
type \eqref{eq:final_bound} based on a case distinction.

\paragraph{Large $\Delta(t)$.} We assume $\Delta(t) \geq 1$.
For every $t> 0$, we may choose $L$ such that
$\Lambda(t) > 1$ and 
proceed to bound \eqref{eq:yurinsky_exponential_moment_updated2}
as
\begin{align*}
     \frac{\Delta(t)}{\Lambda(t)^q}
    +
    \exp
    \left(
        htB
        \left( \Delta(t) / \Lambda(t)^{q-1} - 1 \right)
        +
        \tilde{c}_2 h^2 B^2
        +
        \tilde{c}_q
        (htB)^q
        \Delta(t)
        e^{ \Lambda(t) th B }
    \right) \\
    \leq
    \Delta(t)
    + 
    \exp \left(
    htB
    \left( \Delta(t) - 1 \right)
     +
        \tilde{c}_2 h^2 B^2
        +
        \tilde{c}_q
        (htB)^q
        \Delta(t)
        e^{ \Lambda(t) th B }
    \right) \\
    \leq \Delta(t) + 2
    \leq 3 \Delta(t) + \exp(-c_2 t^2) 
\end{align*}
for any $c_2>0$, 
if we choose $h$ small enough such that
\begin{equation*}
  \exp (
     htB
     \underbrace{
    \left( \Delta(t) - 1 \right)
    }_{\geq 0}
    +
        \tilde{c}_2 h^2 B^2
        +
        \tilde{c}_q
        (htB)^q
        \Delta(t)
        e^{ \Lambda(t) th B }
    ) \leq 2.  
\end{equation*}
This establishes \eqref{eq:final_bound}.



\paragraph{Small $\Delta(t)$.}
We therefore need to consider the case
where $\Delta(t)<1$.
Without loss of generality, in this setting  
we again assume that we are choosing $L$ appropriately to ensure
that 
$\Lambda(t) > 1$. 
This gives us
\begin{equation*}
    \frac{\Delta(t)}{\Lambda(t)^q} \leq \Delta(t),
\end{equation*}
and hence we only
need to bound the exponential term in \eqref{eq:yurinsky_exponential_moment_updated2},
as the polynomial term satisfies
the final bound \eqref{eq:final_bound}.
We distinguish two cases, as
the exponential term in \eqref{eq:yurinsky_exponential_moment_updated2}
behaves differently for
different ratios of $t^2 / \log(1/\Delta(t))$.

\paragraph{Small $\Delta(t)$, case 1:} 
We assume $\Delta(t) < 1$,
$\Lambda(t) > 1$ and
$t^2 \leq 6 \tilde{c}_2 \log(1/\Delta(t))$. 
We will show that in this case, the
exponential term in \eqref{eq:yurinsky_exponential_moment_updated2}
can be bounded by $c_1 \exp(-c_2t^2)$.

We choose the optimization parameter $h := zt/B$
with some $z>0$, which we will determine later.
The exponent in
the right hand side of
\eqref{eq:yurinsky_exponential_moment_updated2}
can now be written as
\begin{align}
    &z t^2 
    \left(
    \Delta(t) / \Lambda(t)^{q-1}\
    -1 
    \right)
    +
    z^2 \tilde{c}_2 t^2
    +
    \tilde{c}_q
    z^q
    t^{2q}
    \Delta(t)
    e^{ \Lambda(t) z t^2 }
    \nonumber \\
    =
    & t^2 
    \left( 
    \Delta(t) / \Lambda(t)^{q-1}
    -z  + \tilde{c}_2z^2
    \right)
    +
    \tilde{c}_q
    z^q
    t^{2q}
    \Delta(t)
    e^{ \Lambda(t) z t^2 } 
    \nonumber\\
    \leq
    &
    t^2 
    \left( 
    z\Delta(t)
    -z  + \tilde{c}_2z^2
    \right)
    +
    \tilde{c}_q
    z^q
    (6 \tilde{c}_2 \log(1 / \Delta(t))^q
    \Delta(t)^{1- 6 \tilde{c}_2\Lambda(t) z} \nonumber \\
    \leq 
    &
    - \underline{c}_1 t^2 
    +
    6^q \tilde{c}_q \tilde{c}_2^q z^q
    e^{-q}
    \left(
    \frac{q}{\underline{c}_2}
    \right)^q,
    \label{eq:exponent_case1}
\end{align}
where we invoke \cref{lem:polylog_exp_bound}
in the last step 
and choose $z$
small enough such that 
we have
\begin{equation*}
- z\Delta(t) + z - \tilde{c}_2 z^2 > c_1 > 0
\quad
\text{and}
\quad
1- 6 \tilde{c}_2\Lambda(t) z > \underline{c}_2 > 0,
\end{equation*} 
in which case
the exponent estimate \eqref{eq:exponent_case1} is bounded by
$- c_1 t^2 + \tilde{c}$, 
yielding the claim.

\paragraph{Small $\Delta(t)$, case 2:} We assume
$\Delta(t) < 1$,
$\Lambda(t) > 1$ and
$t^2 > 6 \tilde{c}_2 \log(1/\Delta(t))$. We will show that in this case, 
the exponential term in \eqref{eq:yurinsky_exponential_moment_updated2}
can be bounded by $c_1 \Delta(t)$.

In this regime, we choose the optimization parameter
$h := z \log(1/\Delta(t)) / (tB) >0$
with some $z > 0$ which we will determine later.
The exponent in \eqref{eq:yurinsky_exponential_moment_updated2}
with this choice becomes 
\begin{align}
   z 
   \underbrace{
   \log(1/\Delta(t))
   }_{>0}
   (
        -1 
        + 
       \Delta(t) / \Lambda(t)^{q-1}
   ) 
   +
   \tilde{c}_2 z^2 t^{-2}
   \underbrace{ 
    \log(1/\Delta(t))^2 
   }_{< t^2 \log(1/\Delta(t)) / (6 \tilde{c}_2)} \nonumber \\
   +
   \tilde{c}_q z^q \log(1/ \Delta(t))^q \Delta(t) e^{\Lambda(t) z 
   \log(1/\Delta(t))} \nonumber \\
   \leq
   z \log(1/\Delta(t))(-1 + \Delta(t))
   +
   \frac{z^2}{6} \log(1/\Delta(t))
   +
   \tilde{c}_q z^q \log(1/ \Delta(t))^q \Delta(t)^{1-  6\tilde{c}_2\Lambda(t)z} \nonumber
   \\
   =
      \underbrace{\log(\Delta(t))}_{< 0}
   ( z - \Delta(t)z - \frac{z^2}{6} ) + 
    \tilde{c}_q z^q \log(1/ \Delta(t))^q \Delta(t)^{1- 6 \tilde{c}_2\Lambda(t)z}
    \nonumber
    \\
    \leq
   \underbrace{\log(\Delta(t))}_{< 0} + 
    \tilde{c}_q z^q 
    e^{-q}
    \left(
    \frac{q}{\underline{c}}
    \right)^q,
    \label{eq:exponent_case2}
\end{align}
where in the last step we
apply \cref{lem:polylog_exp_bound} and
choose $z$ small enough 
such that
\begin{equation*}
    z - \Delta(t)z - \frac{z^2}{6} \geq 1
    \quad
    \text{and}
    \quad
    1 - 6\tilde{c}_2\Lambda(t)z > \underline{c} > 0.
\end{equation*}
We finally obtain an upper bound
of the exponent \eqref{eq:exponent_case2} as
$
    \log(\Delta(t)) + \tilde{c},
$ 
yielding the claim.

\paragraph{Final bound:} 
The final bound \eqref{eq:final_bound} results from
considering all of the 
the above cases and choosing the individual constants
$c_1$ and $c_2$
large enough such that all cases are satisfied simultaneously
(for example by choosing the maximum of the individual constants
obtained above).

}


\subsection{Concentration of empirical covariance operators}
We have the following classical bound on the estimation error of 
the empirical covariance operators
when \cref{prop:hoeffding} is applied to the centered
random operators
$\xi_i := \phi(X_i) \otimes \phi(X_i) - \mathcal{C}_\pi$
which are bounded by $2 \kappa^2$ in 
Hilbert--Schmidt norm under \cref{assump:kernel}.
Note 
that the operator norm can be bounded by the Hilbert--Schmidt norm.

\begin{corollary}[Sample error of empirical covariance operator]
Under \cref{assump:kernel}, we have
\begin{equation*}
    \Norm{ \widehat{\mathcal{C}}_\pi - \mathcal{C}_\pi }_{S_2(\mathcal{H})}
    \leq 2 \kappa^2 \sqrt{ \frac{2 \log(2/\delta)}{n}} 
\end{equation*}
with probability at least $1 - \delta$.
\end{corollary}

We now recall a result by
\cite[Proposition 1]{guo2017learning}.

\begin{proposition}[Empirical inverse, \cite{guo2017learning},
Proposition 1]
\label{prop:concentration_inverse_product}
Let \cref{assump:kernel} hold.
For all $\delta \in (0,1), n \in \mbn, \alpha >0$, denote
\begin{equation}
\label{eq:simplify-inverse-conc}
\cB(n, \delta, \alpha) := 1 + \log^2(2/\delta) \left( \frac{2\kappa^2}{n \alpha} + 
\sqrt{\frac{4 \kappa^2 \cN(\alpha)}{n \alpha}} \right)^2.
\end{equation}
With probability at least $1-\delta$
\begin{equation}
    \label{eq:concentration_inverse_product}
    \Norm{ (\widehat{\mathcal{C}}_\pi + \alpha \idop_{\mathcal{H}})^{-1} 
    (\mathcal{C}_\pi  + \alpha \idop_{\mathcal{H}}) }_{\bounded(\mathcal{H})}
    \leq 2\cB(n, \delta, \alpha) .
\end{equation}
In particular, if 
\begin{equation}
\label{eq:simplify-inverse-conc2}
\frac{1}{n \alpha} \leq \cN(\alpha)\;, \qquad 16\max\{1, \kappa^4\}\log^2(2/\delta) \; \frac{\cN(\alpha)}{\alpha} \leq n,  
\end{equation}
then $\Norm{ (\widehat{\mathcal{C}}_\pi + \alpha \idop_{\mathcal{H}})^{-1} 
    (\mathcal{C}_\pi  + \alpha \idop_{\mathcal{H}}) }_{\bounded(\mathcal{H})} \leq 4$. 
\end{proposition}

\begin{proof}
    We prove the particular case. Note that if $(n\alpha)^{-1} \leq \mathcal{N}(\alpha)$, 
    $$
    \cB(n, \delta, \alpha) \leq 1 + \log^2(2/\delta)(\kappa+1)^2 \frac{4 \kappa^2 \cN(\alpha)}{n \alpha} \leq  1 + 16\max\{1,\kappa^4\}\log^2(2/\delta)\frac{\cN(\alpha)}{n \alpha}.
    $$
    Therefore, if $16\max\{1, \kappa^4\}\log^2(2/\delta) \cN(\alpha) \leq n\alpha$, then $\cB(n, \delta, \alpha) \leq 2$.
\end{proof}

We next give a Corollary, providing a simplified bound under a worst case assumption for 
the effective dimension that can be applied without information about the eigenvalue decay. 

\begin{corollary}[Empirical inverse, worst case scenario]
\label{cor:concentration_inverse_product}
Let \cref{assump:kernel} hold.
Assume that $\cN(\alpha)\leq c\alpha^{-1}$ for some $c>0$. Then, with confidence $1-\delta$, we have 
\[ \Norm{ (\widehat{\mathcal{C}}_\pi + \alpha \idop_{\mathcal{H}})^{-1} 
    (\mathcal{C}_\pi  + \alpha \idop_{\mathcal{H}}) }_{\bounded(\mathcal{H})} \leq 4, \]
    provided that 
\[ C_\kappa\; \frac{\log(2/\delta)}{\alpha} \leq \sqrt{n}, \qquad C_\kappa:=2(1+\sqrt c)\cdot \max\{1, \kappa^2\} \;.  \]
\end{corollary}

\begin{proof}
    Using $\cN(\alpha)\leq c\alpha^{-1}$ and plugging it 
    into \eqref{eq:simplify-inverse-conc}, we obtain
    \[
    \mathcal{B}(n, \delta, \alpha) \leq 1 + \log^2(2/\delta) \left( \frac{2\kappa^2}{n\alpha} + \frac{2 \sqrt{c} \kappa}{\sqrt{n} \alpha} \right)^2.
    \]
Note that
\[
\frac{2\kappa^2}{n\alpha} + \frac{2 \sqrt{c} \kappa}{\sqrt{n} \alpha} \leq \frac{2(1 + \sqrt{c}) \cdot \max\{1, \kappa^2\}}{\sqrt{n} \alpha} = \frac{C_\kappa}{\sqrt{n} \alpha}.
\]
Therefore,
\[
\mathcal{B}(n, \delta, \alpha) \leq 1 + \log^2(2/\delta) \cdot \frac{C_\kappa^2}{n \alpha^2}.
\]
To ensure $\mathcal{B}(n, \delta, \alpha) \leq 2$, it suffices that
\[
\log^2(2/\delta) \cdot \frac{C_\kappa^2}{n \alpha^2} \leq 1 \quad \Longleftrightarrow \quad C_\kappa \cdot \frac{\log(2/\delta)}{\alpha} \leq \sqrt{n}.
\]
Under this condition, we obtain $\mathcal{B}(n, \delta, \alpha) \leq 2$, and hence
\[
\left\| (\widehat{\mathcal{C}}_\pi + \alpha I)^{-1} (\mathcal{C}_\pi + \alpha I) \right\| \leq 4,
\]
as claimed.
\end{proof}

\begin{remark}[Worst case scenario]
    Note that, by \cref{assump:kernel},
    we always have
    \begin{equation*}
        \mathcal{N}(\alpha)
        \leq \tr(\cov) \alpha^{-1}
        = 
        \E[ \norm{\phi(X)}^2_\mathcal{H} ]
        \alpha^{-1}
        \leq
        \kappa^2 \alpha^{-1}.
    \end{equation*}
    We therefore see that the constant $c$
    in the above assumption always
    exists and satisfies
    $c \leq \kappa^2$.
\end{remark}

Under the eigenvalue decay assumption \eqref{eq:evd_new},
the conditions of \cref{prop:concentration_inverse_product}
can be simplified.

\begin{corollary}[Empirical inverse under \eqref{eq:evd_new}]
\label{cor:concentration_inverse_product_under_evd}
Let \cref{assump:kernel} and \eqref{eq:evd_new} hold. 
Then, with confidence $1-\delta$, we have 
\[ \Norm{ (\widehat{\mathcal{C}}_\pi + \alpha \idop_{\mathcal{H}})^{-1} 
    (\mathcal{C}_\pi  + \alpha \idop_{\mathcal{H}}) }_{\bounded(\mathcal{H})} \leq 4, \]
    provided that 
$$
\log(2/\delta)\left( \frac{2\kappa^2}{n\alpha} + \frac{2 \sqrt{\tilde{D}} \kappa}{\sqrt{n} \alpha^{(1 + p)/2}} \right) \leq 1,
$$
where $\tilde{D}$ is defined in \cref{lem:evd_new}.
\end{corollary}

\begin{proof}
    By \cref{lem:evd_new}, under \eqref{eq:evd_new}, there exists a constant $\tilde{D} > 0$ such that $\mathcal{N}(\alpha) \leq \tilde{D} \alpha^{-p}$. 
    Inserting this bound into 
    \eqref{eq:simplify-inverse-conc} leads to
    \[
    \mathcal{B}(n, \delta, \alpha) \leq 1 + \log^2(2/\delta) \left( \frac{2\kappa^2}{n\alpha} + \frac{2 \sqrt{\tilde{D}} \kappa}{\sqrt{n} \alpha^{(1 + p)/2}} \right)^2.
    \]
\end{proof}

We can use \cref{prop:concentration_inverse_product} to bound
the weighted norm of a function in $\mathcal{H}$ by an
empirically weighted and regularized norm.

\begin{lemma}[Concentration of weighted regularized norm]
    \label{lem:concentration_weighted_norm}
    Under \cref{assump:kernel},
    for all $f \in \mathcal{H}$ and all $\alpha >0$ and $s \in [0, 1]$, we have
    \begin{equation}
        \Norm{ \mathcal{C}_\pi^{s} f }_{\mathcal{H}}
        \leq 2^s  \cB(n, \delta, \alpha)^s \;
        \Norm{ 
            (\widehat{\mathcal{C}}_\pi + \alpha \idop_{\mathcal{H}})^{s} f
        }_{\mathcal{H}}
    \end{equation}
    with probability at least $1- \delta$.
\end{lemma}

\begin{proof}
    We have
    \begin{equation}
    \label{eq:concentration_weighted_norm}
        \Norm{ \mathcal{C}_\pi^{s} f }_{\mathcal{H}}
        \leq 
        \Norm{ 
            \mathcal{C}_\pi^{s} (\mathcal{C}_\pi + \alpha \idop_{\mathcal{H}})^{-s} 
        }_{\bounded(\mathcal{H})}
        \Norm{ 
            (\mathcal{C}_\pi + \alpha \idop_{\mathcal{H}})^{s} 
            (\widehat{\mathcal{C}}_\pi + \alpha \idop_{\mathcal{H}})^{-s} 
        }_{\bounded(\mathcal{H})}
        \Norm{ 
            (\widehat{\mathcal{C}}_\pi + \alpha \idop_{\mathcal{H}})^{s} f
        }_{\mathcal{H}}.
    \end{equation}
    We bound the terms on the right hand side individually.
    Applying \cref{lem:cordes} to the first term on the right hand side
    of \eqref{eq:concentration_weighted_norm},
    we have 
    \begin{equation*}
        \Norm{ 
            \mathcal{C}_\pi^{s} (\mathcal{C}_\pi + \alpha \idop_{\mathcal{H}})^{-s} 
        }_{\bounded(\mathcal{H})}
        \leq
        \Norm{ 
            \mathcal{C}_\pi (\mathcal{C}_\pi + \alpha \idop_{\mathcal{H}})^{-1}
        }^s_{\bounded(\mathcal{H})} \leq 1.
    \end{equation*} 
    The second term can be bounded with \cref{prop:concentration_inverse_product}.     
\end{proof}

\section{Miscellaneous results}
\label{app:miscellaneous}

We recall a standard result that addresses
the saturation property of Tikhonov regularization,
see e.g.\ \cite[][Example 4.15]{EHN96}.

\begin{lemma}
\label{lem:qualification}
    Let $\alpha >0$ and $\nu > 0$.
    We have
    \begin{equation*}
        \sup_{t \in [0, \kappa^2]}
        \frac{ \alpha t^\nu}{t + \alpha} 
        \leq \max\{1, {\kappa^{2(\nu - 1)}}\} \,
        \alpha^{\min\{\nu,1\}}.
    \end{equation*}
\end{lemma}

We will also frequently use the following 
classical bound for products of
fractional operators.

\begin{lemma}[Cordes inequality, \citep{Furuta1989}]
\label{lem:cordes}
Let $s \in [0,1]$ and
$A, B \in \bounded(\mathcal{H})$ be
positive-semidefinite selfadjoint.
Then we have
$\norm{A^s B^s}_{\bounded(\mathcal{H})} \leq
    \norm{A B}^s_{\bounded(\mathcal{H})}$.
\end{lemma}

The next result is part of \cite[Proposition~3.5]{Rio2017}.
It allows us to bound moments between $2$ and $q$.
For convenience, we provide a detailed, self-contained proof.

\yurinsky{

%
%
\begin{lemma} \label{lem:fromRio17Prop3.5}
Let $q>2$. For all $\rho\in[2,q]$, $a \in \Rp$, and $x\in\Rp$ we have
\begin{equation}
    (q-2)x^\rho \leq (\rho-2)a^{\rho-q}x^q + (q-\rho)a^{\rho-2}x^2.
\end{equation}
\end{lemma}
\begin{proof}
We first show that 
\begin{equation} \label{eq:fromRio17Prop3.5:1}
    (q-2)\left(\frac{x}{a}\right)^{\rho-2} \leq (\rho-2)\left(\frac{x}{a}\right)^{q-2}  + (q-\rho).
\end{equation}
For this, define for some $y\in\Rp$ and $t\in\Rp$ the function $f: [0,t]\rightarrow \R$ by $f(s)=\exp(\ln(y)s)$.
Note that $f$ is convex (since $f'(s)=f(s)\ln(y)$, $f''(s)=f(s)\ln(y)^2>0$).
For $s\in[0,t]$ we then get
\begin{align*}
    y^s = f(s) & = f((1-s/t)\cdot 0 + s/t \cdot t) \\
        & \leq (1-s/t)f(0) + s/t \cdot f(t) = (1-s/t) + s/t
        \cdot y^t
\end{align*}
and multiplying both sides by $t$ leads to 
\begin{equation*}
    ty^s \leq sy^t + (t-s).
\end{equation*}
Setting $t=q-2$, $s=\rho-2$, and $y=x/a$ then shows \eqref{eq:fromRio17Prop3.5:1}.
Finally, multiplying both sides of this with $x^2/a^{2-\rho}$ establishes the result.
\end{proof}

Finally, we require the following basic inequality.

\begin{lemma}
\label{lem:polylog_exp_bound}
Consider the function $f: (0,1) \to \R_+$
given by
   $ f(\alpha) = \log(1/\alpha)^q \alpha^{s}$
for some $s >0$ and $q >0$.
Then we have
\begin{equation*}
    \sup_{\alpha \in (0,1)} f(\alpha) 
    \leq e^{-q} 
    \left(
        \frac{q}{s}
    \right)^q.
\end{equation*}
\end{lemma}

\begin{proof}
    We may substitute
    the transformation $x := \log(1/\alpha) \in (0, \infty)$
    and consider $f(x) = x^q \exp(-sx)$ instead.
    Note that $f$ is nonnegative on $(0, \infty)$
    and that $f(0) = \lim_{x \to \infty} f(x) = 
    \lim_{x \to 0_+} f(x)= 0$.
    The derivative $f'(x) = \exp(-sx) (q x^{q-1} - s x^{q})$
    has the unique root $x_\star = q/s$ and 
    we hence obtain $f(x_\star) = \exp(-q) (q/s)^q$ 
    as the maximal value.
\end{proof}
    
}

\end{document}